\documentclass{article} 
\usepackage{iclr2023_conference,times}

\usepackage{amsmath,amsfonts,bm}

\def\eqref#1{equation~\ref{#1}}

\def\1{\bm{1}}

\DeclareMathAlphabet{\mathsfit}{\encodingdefault}{\sfdefault}{m}{sl}
\SetMathAlphabet{\mathsfit}{bold}{\encodingdefault}{\sfdefault}{bx}{n}

\newcommand{\E}{\mathbb{E}}

\newcommand{\R}{\mathbb{R}}

\newcommand{\KL}{D_{\mathrm{KL}}}

\usepackage{hyperref}
\usepackage{url}
\usepackage[utf8]{inputenc} 
\usepackage{cleveref}
\usepackage{algorithm}
\usepackage{algpseudocode}
\usepackage{graphicx}
\usepackage{subfigure}
\usepackage{listings}

\usepackage{nicefrac}

\newtheorem{theorem}{Theorem}
\newtheorem{lemma}{Lemma}
\newtheorem{corollary}{Corollary}
\newtheorem{proposition}{Proposition}
\newtheorem{definition}{Definition}
\newtheorem{remark}{Remark}

\newtheorem{assumption}{Assumption}

\newtheorem{proof}{Proof}

\newcommand{\Renyii}{\mathop{\mathrm{D}_{2}}\nolimits}
\newcommand{\Renyi}{\mathop{\mathrm{D}_{\alpha}}\nolimits}

\newcommand{\cP}{{\mathcal P}}
\newcommand{\cO}{{\mathcal O}}
\newcommand{\cN}{{\mathcal N}}
\newcommand{\cH}{{\mathcal H}}
\newcommand{\norm}[1]{\left\| #1 \right\|}
\newcommand{\normsq}[1]{\left\| #1 \right\|^2}
\newcommand{\inner}[2]{\left< #1 , #2 \right>}

\newcommand{\cone}{{\norm{g_n(x)}}}
\newcommand{\ctwo}{{\norm{\nabla g_n(x)}}}
\newcommand{\cthree}{{C_3}}

\newcommand{\cfive}{{C_{\rho_n}\left(\delta\right)}}

\newcommand{\myeqref}[1]{{\color{red}(\ref{#1})}}

\usepackage[colorinlistoftodos,bordercolor=orange,backgroundcolor=orange!20,linecolor=orange,textsize=scriptsize]{todonotes}

\newcommand{\algname}[1]{{\sf \small \color{blue} #1}}
\newcommand{\squeeze}{\textstyle}

\title{Improved Stein Variational Gradient Descent with Importance Weights}

\author{Lukang Sun\\KAUST	 \thanks{King Abdullah University of Science and Technology, Thuwal, Saudi Arabia}\\ \texttt{lukang.sun@kaust.edu.sa} \And Peter Richt\'arik\\KAUST\\	\texttt{peter.richtarik@kaust.edu.sa}}

\begin{document}

	\maketitle
	
	\begin{abstract}
		Stein Variational Gradient Descent~(\algname{SVGD}) is a popular sampling algorithm used in various machine learning tasks. It is well known that \algname{SVGD} arises from a discretization of the kernelized gradient flow of the Kullback-Leibler divergence $\KL\left(\cdot\mid\pi\right)$, where $\pi$ is the target distribution. In this work, we propose to enhance \algname{SVGD} via the introduction of  {\em importance weights}, which leads to a new method for which we coin the name  \algname{$\beta$-SVGD}. In the continuous time and infinite particles regime, the time for this flow to converge to the equilibrium distribution $\pi$, quantified by the Stein Fisher information, depends on $\rho_0$ and $\pi$ very weakly. This is very different from the kernelized gradient flow of Kullback-Leibler divergence, whose time complexity depends on $\KL\left(\rho_0\mid\pi\right)$. Under certain assumptions, we provide a descent lemma for the population limit \algname{$\beta$-SVGD}, which covers the descent lemma for the population limit \algname{SVGD} when $\beta\to 0$. We also illustrate the advantages of \algname{$\beta$-SVGD} over \algname{SVGD} by experiments.
	\end{abstract}

\section{Introduction}
The main technical task of Bayesian inference is to estimate integration with respect to the posterior distribution $$\pi(x)\propto e^{-V(x)},$$ where $V:\R^d\to \R$ is a potential. In practice, this is often reduced to sampling points from the distribution $\pi$. Typical methods that employ this strategy include algorithms based on Markov Chain Monte Carlo~(\algname{MCMC}), such as Hamiltonian Monte Carlo \citep{neal2011mcmc}, also known as Hybrid Monte Carlo~(\algname{HMC}) \citep{duane1987hybrid, betancourt2017conceptual}, and algorithms based on Langevin dynamics \citep{dalalyan2019user,durmus2017nonasymptotic,cheng2018underdamped}. 

One the other hand, Stein Variational Gradient Descent (\algname{SVGD})---a different strategy suggested by \citet{liu2016stein}---is based on an interacting particle system. In the population limit, the interacting particle system can be seen as the kernelized negative gradient flow of the Kullback-Leibler divergence
\begin{equation}\label{eq:kl}
	\squeeze	\KL\left(\rho\mid\pi\right):=\int\log\left(\frac{\rho}{\pi}\right)(x) \;d\rho(x);
\end{equation}
see \citep{liu2017stein,duncan2019geometry}. \algname{SVGD} has already been widely used in a variety of machine learning settings, including
variational auto-encoders \citep{pu2017vae}, reinforcement learning~\citep{liu2017policy}, 
sequential decision making~\citep{zhang2018learning,zhang2019scalable}, generative adversarial networks~\citep{tao2019variational} and federated learning~\citep{kassab2020federated}. However, current theoretical understanding of \algname{SVGD} is limited to its infinite particle version~\citep{liu2017stein,korba2020non,salim2021complexity,sun2022convergence}, and the theory on finite particle \algname{SVGD} is far from satisfactory.

Since \algname{SVGD} is built on a discretization of the kernelized negative gradient flow of~\myeqref{eq:kl}, we can learn about its sampling potential by studying this flow. In fact, a simple calculation reveals that
\begin{equation}\label{eq:errror}
	\squeeze\min \limits_{0 \leq s \leq t} I_{Stein}\left(\rho_{s} \mid \pi\right) \leq \frac{\KL\left(\rho_{0} \mid \pi\right)}{t},
\end{equation}
where $I_{Stein}\left(\rho_s\mid\pi\right)$ is the Stein Fisher information (see Definition~\ref{def:steinfisher}) of $\rho_s$ relative to $\pi$, which is typically used to quantify how close to $\pi$   are the probability distributions $\left(\rho_s\right)_{ s=0}^t$ generated along this flow.  In particular, if our goal is to guarantee 
$
\min \limits_{0 \leq s \leq t} I_{Stein}\left(\rho_{s} \mid \pi\right) \leq \varepsilon,
$
result \myeqref{eq:errror} says that we need to take 
$$\squeeze t\geq \frac{\KL\left(\rho_{0} \mid \pi\right)}{\varepsilon}.$$  Unfortunately, and this is the key motivation for our work, the quantity the initial $\operatorname{KL}$ divergence $\KL\left(\rho_0\mid\pi\right)$ can be very large. Indeed,    it can be proportional to the underlying dimension, which is highly problematic in high dimensional regimes. \cite{salim2021complexity} and \cite{sun2022convergence} have recently derived an iteration complexity bound for the infinite particle \algname{SVGD} method. However, similarly to the time complexity of the continuous flow, their bound depends on $\KL\left(\rho_{0} \mid \pi\right)$.


\subsection{Summary of contributions}

In this paper, we design a family of continuous time flows---which we call \algname{$\beta$-SVGD} flow---by combining {\em importance weights} with the kernelized gradient flow of the $\operatorname{KL}$-divergence. Surprisingly, we prove that the time for this flow to converge to the equilibrium distribution $\pi$, that is $\min_{0\leq s\leq t}I_{Stein}\left(\rho_s\mid\pi\right)\leq\varepsilon$ with $(\rho_s)_{s=0}^t$ generated along \algname{$\beta$-SVGD} flow, can be bounded by $-\frac{1}{\varepsilon\beta(\beta+1)}$ when $\beta\in (-1,0)$. This indicates that the importance weights can potentially accelerate \algname{SVGD}. Actually, we design \algname{$\beta$-SVGD} method based on a discretization of the \algname{$\beta$-SVGD} flow and provide a descent lemma for its population limit version. Some simple experiments in \Cref{apdx:exp} verify our predictions.

We summarize our contributions in the following:
\begin{itemize}
	\item {\bf A new family of flows.} We construct a family of continuous time flows for which we coin the name \algname{$\beta$-SVGD} flows. These flows do  {\em not} arise from a time re-parameterization of the \algname{SVGD} flow since their trajectories are different, nor can they be seen as the kernelized gradient flows of the R{\'e}nyi divergence. 
	\item {\bf Convergence rates.}	When $\beta\to 0$, this returns back to the kernelized gradient flow of the $\operatorname{KL}$-divergence~(\algname{SVGD} flow); when $\beta\in (-1,0)$, the convergence rate of \algname{$\beta$-SVGD} flows is significantly improved than that of the \algname{SVGD} flow in the case $\KL\left(\rho_0\mid\pi\right)$ is large.  Under a Stein Poincar\'e inequality, we derive an exponential convergence rate of $2$-R\'enyi divergence along \algname{$1$-SVGD} flow. Stein Poincar\'e inequality is proved to be weaker than Stein log-Sobolev inequality, however like Stein log-Sobolev inequality, it is not clear to us when it does hold.
	\item {\bf Algorithm.} We design \algname{$\beta$-SVGD} algorithm based on a discretization of the \algname{$\beta$-SVGD} flow and we derive a descent lemmas for the population limit \algname{$\beta$-SVGD}. 
	\item {\bf Experiments.} Finally, we do some experiments to illustrate the advantages of \algname{$\beta$-SVGD} with negative $\beta$. The simulation results on \algname{$\beta$-SVGD}  corroborate our theory.
\end{itemize}

\subsection{Related works}
The \algname{SVGD} sampling technique was first presented in the fundamental work of \cite{liu2016stein}.
Since then, a number of \algname{SVGD} variations have been put out.
The following is a partial list: Newton version \algname{SVGD} \citep{detommaso2018stein}, stochastic \algname{SVGD} \citep{gorham2020stochastic}, mirrored \algname{SVGD} \citep{shi2021sampling}, random-batch method \algname{SVGD} \citep{li2020stochastic} and matrix kernel \algname{SVGD} \citep{wang2019stein}.
The theoretical knowledge of \algname{SVGD} is still constrained to population limit SVGD.
The first work to demonstrate the convergence of \algname{SVGD} in the population limit was by \cite{liu2017stein, korba2020non} then derived a similar descent lemma for the population limit \algname{SVGD} using a different approach. However, their results relied on the path information and thus were not self-contained, to provide a clean analysis, \cite{salim2021complexity} assumed a Talagrand's $T_1$ inequality of the target distribution $\pi$ and gave the first iteration complexity analysis in terms of dimension $d$. Following the work of \cite{salim2021complexity, sun2022convergence}  derived a descent lemma for the population limit \algname{SVGD} under a non-smooth potential $V$.

In this paper, we consider a family of generalized divergences, R\'enyi divergence,  and \algname{SVGD} with importance weights. For these two themes, we name a few  but non-exclusive related results.
\cite{wang2018variational} proposed to use the $f$-divergence instead of $\operatorname{KL}$-divergence in the variational inference problem, here $f$ is a convex function; \cite{yu2020training} also considered variational inference with $f$-divergence but with its dual form. \cite{han2017stein} considered combining importance sampling with \algname{SVGD}, however the importance weights were only used to adjust the final sampled points but not in the iteration of \algname{SVGD} as in this paper; \cite{liu2017black} considered importance sampling, they designed a black-box scheme to calculate the importance weights~(they called them Stein importance weights in their paper) of any set of points.

\section{Preliminaries}\label{sec:pre}
We assume the target distribution $\pi\propto e^{-V}$, and we have oracle to calculate the value of $e^{-V(x)}$ for all $ x\in \R^d$.

\subsection{Notation}
Let $x=\left(x_1,\ldots,x_d\right)^{\top},y=\left(y_1,\ldots,y_d\right)^{\top}\in\R^d$, denote $\inner{x}{y}:=\sum_{i=1}^d x_iy_i$ and $\norm{x}:=\sqrt{\inner{x}{x}}$. For a square matrix $B\in \R^{d\times d}$, the  operator norm and Frobenius norm of $B$ are defined respectively by
$\|B\|_{op} := \sqrt{\varrho (B^{\top}B)}$ and $\|B\|_{F} := \sqrt{\sum_{i=1}^{d} \sum_{j=1}^{d} B_{i,j}^2}$, respectively, 
where $\varrho$ denotes the spectral radius. It is easy to verify that $\norm{B}_{op}\leq \norm{B}_F$. Let $\cP_2(\R^d)$ denote the  space of probability measures with finite second moment; that is, for any $\mu\in\cP_2(\R^d)$ we have $\int \norm{x}^2 \;d\mu(x)<+\infty$. The Wasserstein $2$-distance between $\rho,\mu\in\cP_2(\R^d)$ is defined by
\begin{equation*}
	\squeeze	W_2\left(\rho,\mu\right):=\inf \limits_{\eta\in\Gamma(\rho,\mu)}\sqrt{\int\norm{x-y}^2 \; d\eta(x,y)},
\end{equation*}
where $\Gamma\left(\rho,\mu\right)$ is the set of all joint distributions defined on $\R^d\times\R^d$ having $\rho$ and $\mu$ as  marginals. The push-forward distribution of $\rho\in\cP_2\left(\R^d\right)$ by a map $T:\R^d\to\R^d$, denoted by $T_{\#}\rho$,  is defined as follows: for any measurable set $\Omega\in\R^d$, $T_{\#}\rho\left(\Omega\right):=\rho\left(T^{-1}\left(\Omega\right)\right)$. By definition of the push-forward distribution, it is not hard to verify that the probability densities satisfy $T_{\#}\rho(T(x))|\operatorname{det} \operatorname{D}T(x)|=\rho(x)$, where $\operatorname{D} T$ is the Jacobian matrix of $T$. The reader can refer to \cite{villani2008optimal} for more details.
\subsection{R\'enyi divergence}
Next, we define the R\'enyi divergence which plays an important role in information theory and many other areas such as hypothesis testing~\citep{morales2013renyi} and multiple source adaptation~\citep{mansour2012multiple}.
\begin{definition}[R{\'e}nyi divergence]
	For two probability distributions $\rho$ and $\mu$ on $\R^d$ and $\rho\ll\mu$, the R{\'e}nyi divergence of positive order $\alpha$ is defined as
	\begin{equation}
		\Renyi(\rho\mid\mu):=\begin{cases}
			\frac{1}{\alpha-1}\log\left(\int\left(\frac{\rho}{\mu}\right)^{\alpha-1}(x) \;d\rho(x)\right)& 0<\alpha<\infty, \; \alpha\neq 1\\
			\int\log\left(\frac{\rho}{\mu}\right)(x) \;d\rho(x)& \alpha=1
		\end{cases}.
	\end{equation}
	If $\rho$ is not absolutely continuous with respect to $\mu$, we set $\Renyi(\rho\mid\mu)=\infty$. Further, we denote $\KL\left(\rho\mid\mu\right):=\mathrm{D}_1\left(\rho\mid\mu\right)$.
\end{definition}
R\'enyi divergence is non-negative, continuous and non-decreasing in terms of the parameter $\alpha$; specifically, we have  $\KL\left(\rho\mid\mu\right)=\lim_{\alpha\to 1}\Renyi(\rho\mid\mu)$. More properties of R\'enyi divergence can be found in a comprehensive article by \cite{van2014renyi}. Besides R\'enyi divergence, there are other generalizations of the $\operatorname{KL}$-divergence, e.g., admissible relative entropies~\citep{arnold2001convex}.

\subsection{Background on \algname{SVGD}}

Stein Variational Gradient Descent~(\algname{SVGD}) is defined on a Reproducing Kernel Hilbert Space~(RKHS) $\cH_0$ with a non-negative definite reproducing kernel $k:\R^d\times\R^d\to \R^{+}$. The key feature of this space is its reproducing property:
\begin{equation}\label{eq:reproduce}
	f(x)=\inner{f(\cdot)}{k(x,\cdot)}_{\cH_0},\qquad \forall f\in\cH_0,
\end{equation}
where $\inner{\cdot}{\cdot}_{\cH_0}$ is the inner product defined on $\cH_0$. Let $\cH$ be the $d$-fold Cartesian product of $\cH_0$. That is, $f\in\cH$ if and only if there exist $f_1, \cdots,f_d\in\cH_0$ such that $f=\left(f_1,\ldots,f_d\right)^{\top}$. Naturally, the inner product on $\cH$ is given by
\begin{equation}
	\squeeze	\inner{f}{g}_{\cH}:=\sum \limits_{i=1}^d\inner{f_i}{g_i}_{\cH_0},\qquad f=\left(f_1,\ldots,f_d\right)^{\top}\in\cH, \qquad g=\left(g_1,\ldots,g_d\right)^{\top}\in\cH.
\end{equation}
For more details of RKHS, the readers can refer to \cite{berlinet2011reproducing}.

It is well known~(see for example \cite{ambrosio2008gradient}) that $\nabla\log\left(\frac{\rho}{\pi}\right)$ is the Wasserstein gradient of $\KL\left(\cdot\mid\pi\right)$ at $\rho\in\cP_2(\R^d)$. \cite{liu2016stein} proposed a kernelized Wasserstein gradient of the $\operatorname{KL}$-divergence, defined by 
\begin{equation}\label{eq:wstgd}
	\squeeze	g_{\rho}(x):=\int k(x,y)\nabla\log\left(\frac{\rho}{\pi}\right)(y) \;d\rho(y)\in\cH.
\end{equation}
Integration by parts yields 
\begin{equation}\label{eq:svgdfloww}
	\squeeze		g_{\rho}(x)=-\int \left[\nabla\log\pi(y)k(x,y)+\nabla_yk(x,y)\right] \;d\rho(y).
\end{equation}
Comparing the Wasserstein gradient $\nabla\log\left(\frac{\rho}{\pi}\right)$ with \myeqref{eq:svgdfloww}, we find that the latter  can be easily approximated by 
\begin{equation}\label{eq:approximatemethods}
	\squeeze		g_{\rho}(x)\approx \hat{g}_{\hat{\rho}} := -\frac{1}{N}\sum \limits_{i=1}^N\left[\nabla\log\pi(x_i)k(x,x_i)+\nabla_{x_i}k(x,x_i)\right],
\end{equation}
with $\hat{\rho}=\frac{1}{N}\sum_{i=1}^N\delta_{x_i}$ and $\left(x_i\right)_{i=1}^N$ sampled from $\rho$. With the above notations, the \algname{SVGD} update rule
\begin{equation}
	\squeeze	x_i\leftarrow x_i+\frac{\gamma}{N} \sum\limits_{j=1}^N \left[\nabla\log\pi(x_j)k(x_i,x_j)+\nabla_{x_j}k(x_i,x_j)\right],\quad i=1,\ldots,N,
\end{equation}
where $\gamma$ is the step-size, can be presented in the compact form
$\hat{\rho}\leftarrow\left(I-\gamma\hat{g}_{\hat{\rho}}\right)_{\#}\hat{\rho}.$ When we talk about the infinite particle \algname{SVGD}, or population limit \algname{SVGD}, we mean
$
\rho\leftarrow\left(I-\gamma{g}_{\rho}\right)_{\#}\rho.
$
The metric used in the study of \algname{SVGD} is the Stein Fisher information or the Kernelized Stein Discrepancy~(KSD).
\begin{definition}[Stein Fisher Information]\label{def:steinfisher}
	Let $\rho \in \mathcal{P}_{2}(\R^d)$. The Stein Fisher Information of $\rho$ relative to $\pi$  is defined by 
	\begin{equation}
		\squeeze		I_{Stein}(\rho \mid \pi):=\iint k(x,y)\inner{\nabla\log\left(\frac{\rho}{\pi}\right)(x)}{\nabla\log\left(\frac{\rho}{\pi}\right)(y)} \;d\rho(x) \;d\rho(y).
	\end{equation}
\end{definition}
A sufficient condition under which  $\lim_{n\to\infty}I_{Stein}(\rho_n\mid\pi)$ implies $\rho_n\to\pi$ weakly can be found in \cite{gorham2017measuring}, which requires: i) the kernel $k$ to be in the form $k(x,y)=\left(c^2+\normsq{x-y}\right)^{\theta}$ for some $c>0$ and $\theta\in (-1,0)$; ii) $\pi\propto e^{-V}$ to be distant dissipative; roughly speaking, this requires $V$ to be convex outside a compact set, see \cite{gorham2017measuring} for an accurate definition.

In the study of the kernelized Wasserstein gradient \myeqref{eq:svgdfloww} and its corresponding continuity equation
\begin{equation*}
	\squeeze	\frac{\partial\rho_t}{\partial t}+\operatorname{div}\left(\rho_tg_{\rho_t}\right)=0,
\end{equation*}
\cite{duncan2019geometry} introduced the following kernelized log-Sobolev inequality to prove the exponential convergence of $\KL\left(\rho_t\mid\pi\right)$ along the direction \myeqref{eq:svgdfloww}:
\begin{definition}[Stein log-Sobolev inequality]
	We say $\pi$ satisfies the Stein log-Sobolev inequality with constant $\lambda>0$ if
	\begin{equation}\label{eq:sli}
		\squeeze	\KL(\rho \mid \pi) \leq \frac{1}{2 \lambda} I_{Stein}(\rho \mid \pi).
	\end{equation}
\end{definition}
While this inequality can guarantee an exponential convergence rate of $\rho_t$ to $\pi$, quantified by the $\operatorname{KL}$-divergence, the condition for $\pi$ to satisfy the Stein log-Sobolev inequality is very restrictive. In fact, little is known about  when \myeqref{eq:sli} holds.

\section{Continuous time dynamics of the \algname{$\beta$-SVGD} flow}\label{sec:conti}
In this section, we mainly focus on the continuous time dynamics of the \algname{$\beta$-SVGD} flow. Due to page limitation, we leave all of the proofs to \Cref{apdx:lemmas}.

\subsection{\algname{$\beta$-SVGD} flow}\label{sec:3.1}
In this paper, a {\em flow} refers to some time-dependent vector field $v_t:\R^d\to\R^d$. This time-dependent vector field will influence the mass distribution on $\R^d$ by the continuity equation or the equation of conservation of mass 
\begin{equation}\label{eq:conservationofmass}
	\squeeze	\frac{\partial\rho_t}{\partial t}+\operatorname{div}\left(\rho_tv_t\right)=0,
\end{equation}
readers can refer to \cite{ambrosio2008gradient} for more details.

\begin{definition}[\algname{$\beta$-SVGD} flow] \label{def:betaflow}Given a weight parameter $\beta \in (-1,+\infty)$, the \algname{$\beta$-SVGD} flow is given by 
	\begin{equation}\label{eq:flowbeta}
		\squeeze	v_t^{\beta}(x):=
		-\left(\frac{\pi}{\rho_t}\right)^{\beta}(x)\int k(x,y)\nabla\log\left(\frac{\rho_t}{\pi}\right)(y)\; d\rho_t(y).
	\end{equation}
	Note that when $\beta=0$, this is the negative kernelized Wasserstein gradient \myeqref{eq:wstgd}. 
\end{definition}

Note that we can not treat \algname{$\beta$-SVGD} flow as the kernelized Wasserstein gradient flow of the 
$(\beta+1)$-R\'enyi divergence. However, they are closely related, and we can derive the following theorem.
\begin{theorem}[Main result]\label{thm:betaflow}
	Along the \algname{$\beta$-SVGD} flow \myeqref{eq:flowbeta}, we have\footnote{In fact, in the proof in \Cref{apdx:lemmas} we know a stronger result. When $\beta\in (-1,0)$, the right hand side of \myeqref{eq:mainresults} is only weakly dependent on $\rho_0$ and $\pi$ and should be $\frac{\left|e^{\beta\mathrm{D}_{\beta+1}\left(\rho_0\mid\pi\right)}-e^{\beta\mathrm{D}_{\beta+1}\left(\rho_T\mid\pi\right)}\right|}{T|\beta(\beta+1)|}$, which is less than $-\frac{1}{T\beta(\beta+1)}$.}
	\begin{equation}\label{eq:mainresults}
		\min_{t\in [0,T]}I_{Stein}\left(\rho_{t}\mid\pi\right)\leq\frac{1}{T}\int_{0}^TI_{Stein}(\rho_t\mid\pi)dt\leq \begin{cases}
			\frac{e^{\beta\mathrm{D}_{\beta+1}\left(\rho_0\mid\pi\right)}}{T\beta(\beta+1)}&\beta>0\\
			\frac{\KL(\rho_0\mid\pi)}{T}&\beta=0\\
			-\frac{1}{T\beta(\beta+1)}&\beta\in(-1,0)
		\end{cases}.
	\end{equation}
\end{theorem}
\begin{figure}
	\centering
	\includegraphics[scale=0.31]{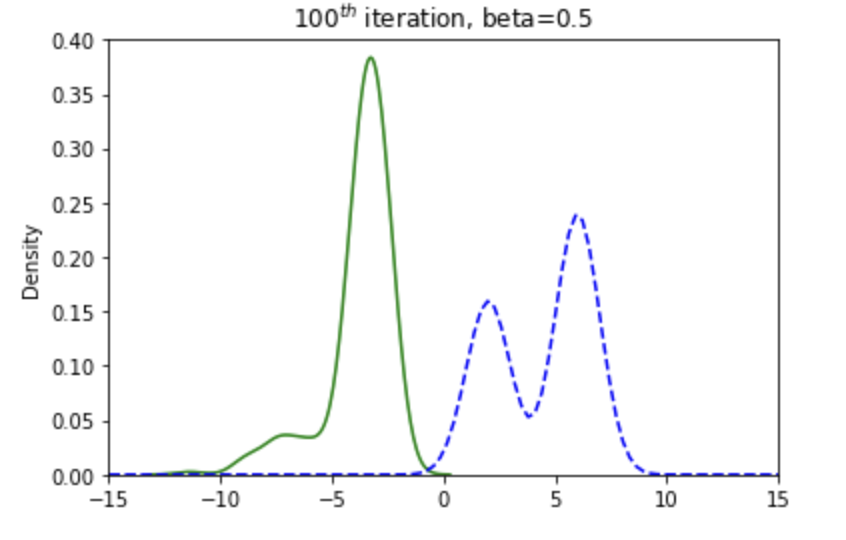}
	\includegraphics[scale=0.31]{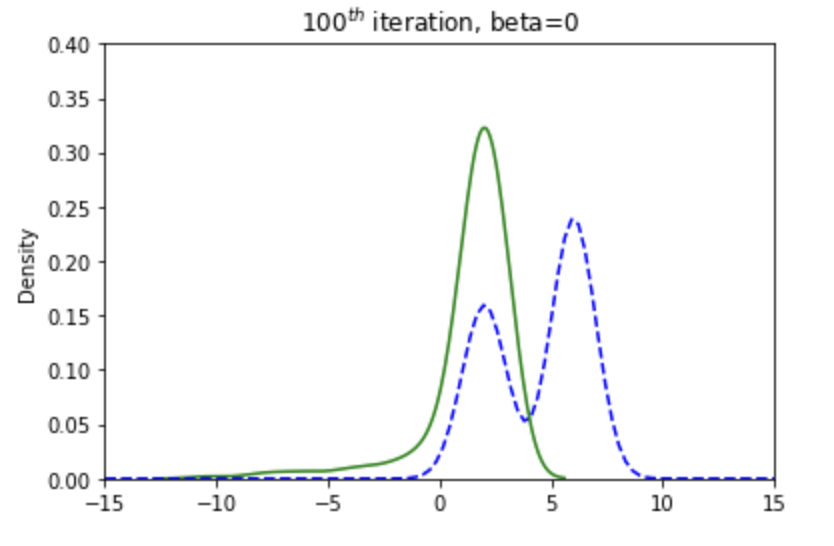}
	\includegraphics[scale=0.31]{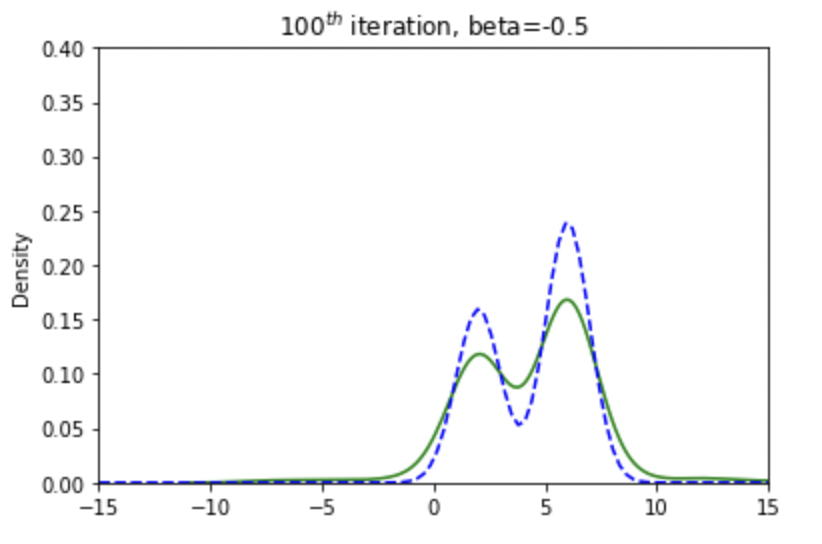}
	\caption{The performance of \algname{$\beta$-SVGD} with $\beta=0.5,0,-0.5$ from left to right, but with the same step-size. The blue dashed line is the target distribution $\pi$: the Gaussian mixture $\frac{2}{5}\cN(2,1)+\frac{3}{5}\cN(6,1)$. The green solid line is the distribution generated by  \algname{$\beta$-SVGD} after 100 iterations; More experiments can be found in \Cref{apdx:exp}.}
	\label{fig:3}
\end{figure}

Note the left hand side of \myeqref{eq:mainresults} is the Stein Fisher information. When $\beta$ decreases from positive to negative, the right hand side of \myeqref{eq:mainresults} changes dramatically; it appears to be independent of $\rho_0$ and $\pi$.  If we do not know the R\'enyi divergence between $\rho_0$ and $\pi$, it seems the best convergence rate is obtained by setting $\beta=-\frac{1}{2}$, that is 
\[
\squeeze		\min_{t\in [0,T]}I_{Stein}\left(\rho_{t}\mid\pi\right)\leq\frac{1}{T}\int_{0}^TI_{Stein}(\rho_t\mid\pi)dt\leq\leq\frac{4}{T}.
\]
It is somewhat unexpected to observe that the time complexity is independent of $\rho_0$ and $\pi$, or to be more precise, that it relies only very weakly on $\rho_0$ and $\pi$ when $\beta\in (-1,0)$. We wish to stress that this is {\em not} achieved by time re-parameterization. In the proof of \Cref{thm:betaflow}, we can see the term $\left(\nicefrac{\pi}{\rho_t}\right)^{\beta}$ in \algname{$\beta$-SVGD} flow~\myeqref{eq:flowbeta} is utilized to cancel term $\left(\nicefrac{\rho_t}{\pi}\right)^{\beta}$ in the Wasserstein gradient of $(\beta+1)$-R\'enyi divergence. Actually, when $\beta\in (-1,0)$, this term has an added advantage and can be seen as the acceleration factor in front of the kernelized Wasserstein gradient of $\operatorname{KL}$-divergence. Specifically, the negative kernelized Wasserstein gradient of $\operatorname{KL}$-divergence $v_t^{0}(x):=-\int k(x,y)\nabla\log(\frac{\rho_t}{\pi})(y)d\rho_t(y)$ is the vector field that compels $\rho_t$ to approach $\pi$, while $\left(\nicefrac{\pi}{\rho_t}\right)^{\beta}(x)$ is big~(roughly speaking this means $x$ is close to the mass concentration region of $\rho_t$ but away from the one of $\pi$), this factor will enhance the vector field at point $x$ and force the mass around $x$ move faster towards the mass concentration region of $\pi$; on the other hand, if $\left(\nicefrac{\pi}{\rho_t}\right)^{\beta}(x)$ is small~(this means $x$ is already near to the mass concentration region of $\pi$), this factor will weaken the vector field and make the mass surrounding $x$ remain within the mass concentration region of $\pi$. This is the intuitive justification for why, when $\beta\in (-1,0)$, the time complexity for \algname{$\beta$-SVGD} flow to diminish the Stein Fisher information only depends on $\rho_0$ and $\pi$ very weakly.


\begin{remark}\label{rmk:111}
	While it may seem reasonable to suspect that the time complexity of the \algname{$\beta$-SVGD} flow with $\beta\leq -1$ will also depend on $\rho_0$ and $\pi$ very weakly, surprisingly, this is not true. In fact,  we can prove that~(see \Cref{apdx:lemmas})
	\begin{equation*}
		\squeeze		\min_{t\in [0,T]}I_{Stein}\left(\rho_{t}\mid\pi\right)\leq\frac{1}{T}\int_{0}^TI_{Stein}(\rho_t\mid\pi)dt\leq \frac{e^{(-\beta-1)\mathrm{D}_{-\beta}\left(\pi\mid\rho_0\right)}}{|T\beta(\beta+1)|}.
	\end{equation*}
	Letting $\beta\to -1$, we get $	\min_{t\in [0,T]}I_{Stein}\left(\rho_{t}\mid\pi\right)\leq\frac{\KL\left(\pi\mid\rho_0\right)}{T}$. The regime when $\beta\leq -1$ is similar to the  $\beta>0$ regime  in \Cref{thm:betaflow}, which heavily depends on $\rho_0$ and $\pi$. 	Mathematically speaking, the weak dependence on $\rho_0$ and $\pi$ is caused by the concavity of the function $s^{\alpha}$ on $s\in\mathbb{R}^+$ when $\alpha=\beta+1\in (0,1)$.
\end{remark}

\subsection{\algname{1-SVGD} flow and the Stein Poincar\'e inequality}

Functional $\KL\left(\cdot\mid\cdot\right)$ is non-symmetric; that is, $\KL\left(\cdot\mid\pi\right)\neq \KL\left(\pi\mid\cdot\right)$, and so is their Wasserstein gradient. The Wasserstein gradient of $\KL\left(\pi\mid\cdot\right)$ at distribution $\rho\in\cP_2(\R^d)$ is $-\nabla\frac{\pi}{\rho}$~(see \Cref{apdx:calculus}), or, to put it another way, $\frac{\pi}{\rho}\nabla\log(\frac{\rho}{\pi})$, which may be regarded as the non-kernelized \algname{$1$-SVGD} flow~(module a minus sigh) when compared to \myeqref{eq:flowbeta}. To conclude, the \algname{$1$-SVGD} flow 
\begin{equation}\label{eq:flowone}
	\squeeze	v_t^{1}(x):=
	-\frac{\pi}{\rho_t}(x)\int k(x,y)\nabla\log\left(\frac{\rho_t}{\pi}\right)(y) \;d\rho_t(y),
\end{equation}
is the negative kernelized  Wasserstein gradient flow of $\KL\left(\pi\mid\cdot\right)$. Next, we will
study the exponential convergence of $2$-R\'enyi divergence along \algname{$1$-SVGD} flow under the Stein Poincar\'e inequality. 
\begin{definition}[Stein Poincar{\'e} inequality]
	We say that $\pi$ satisfies  the Stein Poincar{\'e} inequality with constant $\lambda>0$ if
	\begin{equation}\label{eq:spi}
		\squeeze		\int|g|^{2} \; d \pi\leq	\frac{1}{\lambda}\iint k(x, y)\inner{\nabla g(x)}{\nabla g(y)} \; d \pi(x) \; d \pi(y),
	\end{equation}
	for any smooth $g$ with $\int g \; d \pi=0$. 
\end{definition}
While \citet{duncan2019geometry} also introduced the Stein Poincar\'e inequality, they presented it in a different form. Just as Poincar\'e inequality is a linearized log-Sobolev inequality~(see for example \cite[Proposition 5.1.3]{bakry2014analysis}), Stein Poincar\'e inequality is also a linearized Stein log-Sobolev inequality~\myeqref{eq:sli}. Although  Stein Poincar\'e inequality is weaker than Stein log-Sobolev inequality, the condition for it to hold is quite restrictive, as in the case of Stein log-Sobolev inequality; see the discussion in \cite[Section 6]{duncan2019geometry}. 
\begin{lemma}[Stein log-Sobolev implies Stein Poincar{\'e}]\label{lem:1}
	If $\pi$ satisfies the Stein log-Sobolev inequality~\myeqref{eq:sli} with constant $\lambda>0$, then it also satisfies the Stein Poincar{\'e} inequality with the same constant $\lambda$.
\end{lemma}
While the proof of the above lemma is a routine task, for completeness we provide it in  \Cref{apdx:lemmas}. The following theorem is inspired by \cite{cao2019exponential}, in which they proved the exponential convergence of R\'enyi divergence along Langevin dynamic under a strongly convex potential $V$. However, due to the structure of \algname{$1$-SVGD} flow, we can only prove the results for $\alpha$-R\'enyi divergence with $\alpha\in (0,2]$.

\begin{theorem}\label{thm:flowone}
	Suppose $\pi$ satisfies the Stein Poincar{\'e} inequality with constant $\lambda>0$. Then the flow \myeqref{eq:flowone} satisfies 
	\begin{equation}\label{eq:exponential2}
		\Renyii\left(\rho_{t} \mid \pi\right) \leq C \cdot \Renyii\left(\rho_{0} \mid \pi\right) \cdot e^{-2 \lambda t},
	\end{equation}
	where $C=\frac{e^{\Renyii\left(\rho_{0} \mid \pi\right)}-1}{\Renyii\left(\rho_{0} \mid \pi\right)}$. 
\end{theorem}

Since $\mathrm{D}_{\alpha_1}\left(\rho\mid\pi\right)\leq \mathrm{D}_{\alpha_2}\left(\rho\mid\pi\right)$ for any $0<\alpha_1\leq\alpha_2<\infty$, the exponential convergence of $\alpha$-R\'enyi divergence with $\alpha\in (0,2)$ can be easily deduced from \myeqref{eq:exponential2}.

\begin{corollary}\label{cor:renyicor}
	Suppose $\pi$ satisfies the Stein Poincar{\'e} inequality with constant $\lambda>0$.  Then the flow \myeqref{eq:flowone} satisfies 
	\begin{equation}\label{eq:exponentialalpha}
		\Renyi\left(\rho_{t} \mid \pi\right) \leq C \cdot \Renyi\left(\rho_{0} \mid \pi\right) \cdot e^{-2 \lambda t}
	\end{equation}
	for all  $\alpha\in (0,2]$, where $C=\frac{e^{\Renyii\left(\rho_{0} \mid \pi\right)}-1}{\Renyi\left(\rho_{0} \mid \pi\right)}$.
\end{corollary}

\section{The \algname{$\beta$-SVGD} algorithm}\label{sec:algorithm}

The \algname{$\beta$-SVGD} algorithm\footnote{For simplicity, we will often just call it \algname{$\beta$-SVGD}; not to  be confused with the \algname{$\beta$-SVGD} flow.} proposed here  is a sampling method suggested by the discretization of the \algname{$\beta$-SVGD} flow~\myeqref{eq:flowbeta}. Our method reverts to the traditional \algname{SVGD} algorithm when $\beta=0$. 

As in \algname{SVGD}, the integral term in the \algname{$\beta$-SVGD} flow~\myeqref{eq:flowbeta} can be approximated by \myeqref{eq:approximatemethods}. However,  when $\beta\neq 0$, we have to estimate the extra importance weight term $\left(\nicefrac{\pi}{\rho_t}\right)^{\beta}$. {	Due to the lack of the normalization constant of $\pi$ and the curse of dimension, we can hardly to use  the kernel density estimation~\citep{silverman2018density} to approximate $\nicefrac{\pi}{\rho_t}$ accurately in high dimension. Here, we use a different approach to approximate $\nicefrac{\pi}{\rho_t} $, known as the Stein importance weight~\citep{liu2017black}. This method does not rely on the normalization constant of $\pi$ and can be scaled to high dimension. Given $N$ points $\left(x_i\right)^{N}_{i=1}$ sampled from $\rho_t$, a non-negative definite reproducing kernel $k$~(can be different from the one in \algname{$\beta$-SVGD}) and the score function $\nabla\log(\pi)=-\nabla V$, the Stein importance weight $\hat{\boldsymbol{w}}\in \mathbb{R}^d_{+}$ is the solution of the following constrained quadratic optimization problem:
	\begin{equation}\label{eq:optimi}
		\underset{\boldsymbol{w}}{\arg \min }\left\{\frac{1}{2}\boldsymbol{w}^{\top} \boldsymbol{K}_{\pi} \boldsymbol{w}, \quad \text { s.t. } \quad \sum_{i=1}^N w_i=1, \quad w_i \geq 0\right\},
	\end{equation}
	where $\boldsymbol{K}_{\pi}:=\left\{k_{\pi}(x_i,x_j)\right\}_{i,j=1}^{N}$  and 
	\begin{equation}
		\begin{aligned}
			{k}_{\pi}(x,y)&:=k(x,y)\inner{\nabla V(x)}{\nabla V(y)}-\inner{\nabla V(x)}{\nabla_y k(x,y)}\\
			&\quad-\inner{\nabla V(y)}{\nabla_x k(x,y)}+\operatorname{tr}\left(\nabla_x\nabla_yk(x,y)\right).
		\end{aligned}
	\end{equation}
	It can be proved that as $N\to+\infty$, $N\hat{\boldsymbol{w}}$ will approximate $\left(\nicefrac{\pi}{\rho_t}\right)$, see \citet[Theorem 3.2.]{liu2017black}. 
	Problem \eqref{eq:optimi} can be solved efficiently by mirror descent with step-size $r$, which can be simplified into the following:
	\begin{equation}\label{eq:mirror}
		\omega_i^{s+1}=\frac{\omega_i^s e^{-r\sum_{j=1}^Nk_{\pi}(x_i,x_j)\omega_j^s}}{\sum_{l=1}^n \omega_l^s e^{-r\sum_{j=1}^Nk_{\pi}(x_l,x_j)\omega_j^s}}, \quad i=1,2, \ldots, N.
	\end{equation}
	With matrix $\boldsymbol{K}_{\pi}$, the computation cost of mirror descent to find the optimum with $\varepsilon$-accuracy is $O(\nicefrac{N^2}{\varepsilon})$, which is independent of dimension $d$. In general, $N$ cannot be too large because the cost of one iteration of  \algname{SVGD} is $O(N^2d)$, which quadratically depends on $N$.
	
}
{\begin{remark}
		Stein matrix $\boldsymbol{K}_{\pi}$ can be efficiently constructed using simple matrix operation, since $\left\{\nabla V(x_i)\right\}_{i=1}^N$ have already been computed in the \algname{SVGD} update.
	\end{remark}
	
	\begin{remark}
		In \Cref{alg:betasvgd}, we replace $\left(\nicefrac{\pi}{\rho_t}\right)^{\beta}(x_i)$ by $\left(\max\left(N\hat{\boldsymbol{w}}_i,\tau\right)\right)^{\beta}$, here $\tau$ is a small positive number to separate 
		$N\hat{\boldsymbol{w}}_i$ from $0$. As explained in \Cref{sec:3.1}, the benefits of  $\left(N\hat{\boldsymbol{w}}\right)^{\beta}$ are twofold: it accelerates points with small weights and stabilizes  points with big weights, and these two advantages are observed in our experiments in \Cref{apdx:exp}.
	\end{remark}
	
}

\begin{algorithm}[h!]\footnotesize
	\caption{$\beta$-Stein Variational Gradient Descent~(\algname{$\beta$-SVGD})}\label{alg:betasvgd}
	\begin{algorithmic}[1]
		\State {\bfseries Input:} A set of initial particles  $(x^0_i)_{i=1}^N$, initial importance weight $\omega_i=\nicefrac{1}{N}$, iteration number $n$ for \algname{$\beta$-SVGD} update, iteration number $m$ for mirror descent update and inner loop number $g$.
		\For{$l=0,1,\ldots,n$}
		\If{$l\operatorname{mod}g\equiv 0$}
		\State $\omega_i^0=\omega_i, \quad i=1,2,\ldots,N$
		\For{$s=0,1,\ldots,m$}
		\State Update $\left\{\omega_i^{s+1}\right\}_{i=1}^N$ by \eqref{eq:mirror} with step-size $r$
		\EndFor
		\State Set $\omega_i=\omega_i^{m+1},\quad i=1,2,\ldots,N$
		\EndIf
		\State Update particles with step-size $\gamma$ and small gap $\tau$: $x^{l+1}_i\leftarrow x^l_i+\gamma\left(\max\left(N\omega_i,\tau\right)\right)^{\beta}\sum_{j=1}^N\left[-k(x_i^l,x_j^l)\nabla_{x_j^l} V(x_j^l)+\nabla_{x_j^l}k(x^l_i,x^l_j)\right],\quad i=1,\ldots,N$
		\EndFor
		\State {\bfseries Return:} Particles $(x^{n+1}_i)_{i=1}^N$.
	\end{algorithmic}
\end{algorithm}

\subsection{Non-asymptotic analysis for \algname{$\beta$-SVGD}}
In this section, we study the convergence of the population limit \algname{$\beta$-SVGD}, that is
\begin{equation}
	x_{n+1}=x_n-\gamma\left[\left(\frac{\pi}{\rho_n}\right)^{\beta}(x_n)\wedge M\right]\int k(x_n,y)\nabla\log\left(\frac{\rho_n}{\pi}\right)(y)d\rho_n(y),
\end{equation}
where  $$\left(\frac{\pi}{\rho_n}\right)^{\beta}(x_n)\wedge M=\lim_{N\to\infty}\left(\max\left(N\omega_i,\tau\right)\right)^{\beta}$$ and $M:=\frac{1}{\tau^{\beta}}$.

Specifically, we establish a descent lemma for it. The derivation of the descent lemma is based on several assumptions. 

The first assumption postulates $L$-smoothness of $V$; this is typically assumed in the study of optimization algorithms, Langevin algorithms and \algname{SVGD}.
\begin{assumption}[$L$-smoothness]\label{asp:4}
	The potential function $V$ of the target distribution $\pi\propto e^{-V}$ is $L$-smooth; that is, \begin{equation*}
		\norm{\nabla^2 V}_{op}\leq L.
	\end{equation*}
\end{assumption}
Our  second assumption postulates two bounds involving the reproducing kernel $k(\cdot,\cdot)$, and is also common when studying \algname{SVGD}; see \citep{liu2017stein,korba2020non,salim2021complexity,sun2022convergence}.
\begin{assumption}\label{asp:2}
	Kernel $k$ is continuously differentiable and there exists $B>0$ such that    $\|k(x, .)\|_{\mathcal{H}_{0}} \leq B$ and 
	\begin{equation*}
		\squeeze	\left\|\nabla_{x} k(x, .)\right\|_{\mathcal{H}}^2=\sum \limits_{i=1}^{d}\left\|\partial_{x_{i}} k\left(x, .\right)\right\|_{\mathcal{H}_{0}}^{2} \leq B^2, \qquad \forall x \in \R^d.
	\end{equation*}
\end{assumption}
By the reproducing property~\myeqref{eq:reproduce}, this is equivalent to $k(x,x)\leq B^2$ and $\sum_{i=1}^d\partial_{x_i}\partial_{y_i}k(x,y)\mid_{y=x}\leq B^2$ for any $x\in \R^d$,  and this is easily satisfied by kernel of the form $k(x,y)=f(x-y)$, where $f$ is some smooth function at point $0$.

The third assumption was already used by \cite{liu2017stein}, and was later  replaced by \cite{salim2021complexity} it with a Talagrand inequality~(Wasserstein distance can be upper bounded by $\operatorname{KL}$-divergence) which depends on $\pi$ only. However, \algname{$\beta$-SVGD} reduces the R\'enyi divergence instead of the $\operatorname{KL}$-divergence. Since we do not have a comparable inequality for the R\'enyi divergence, we are forced to adopt the one from \citep{liu2017stein} here.
\begin{assumption}\label{asp:3}
	There exists $C>0$ such that $\sqrt{I_{Stein}\left(\rho_n\mid\pi\right)}\leq C$ for all $n=0,1,\ldots,N$.
\end{assumption}
In the proof of the descent lemma, the next two assumptions help us deal with the extra term $\left(\nicefrac{\pi}{\rho_n}\right)^{\beta}$. Note that the fourth assumption is very weak. In fact, as long as $Z_n(x,y)\rho_n(x)\rho_n(y)$
is integrable on $\R^d\times\R^d$, then by the monotone convergence theorem, the truncating number $M_{\rho_n}(\delta)$ is always attainable since $\left(\nicefrac{\rho_n}{\pi}\right)^{\beta}\left[ \left(\nicefrac{\pi}{\rho_n}\right)^{\beta}\wedge M\right]$ is non-decreasing and converges point-wise to $1$ as $M\to+\infty$.

\begin{assumption}\label{asp:1}
	For any small $\delta>0$, we can find $M_{\rho_n}(\delta)>0$ such that 
	\begin{equation}
		\squeeze	\left|I_{Stein}\left(\rho_n\mid\pi\right)-\iint \left(\frac{\rho_n}{\pi}\right)^{\beta}(x) \left[\left(\frac{\pi}{\rho_n}\right)^{\beta}(x)\wedge M_{\rho_n}(\delta)\right] Z_n(x,y) \;d\rho_n(x) \;d\rho_n(y)\right|\leq\delta,
	\end{equation}
	where  $Z_n(x,y) := k(x,y)\inner{\nabla\log\left(\frac{\rho_n}{\pi}\right)(x)}{\nabla\log\left(\frac{\rho_n}{\pi}\right)(y)}$.	
\end{assumption}

Our fifth and last assumption is of a technical nature, and helps us bound $\norm{\nabla_x \left(\nicefrac{\pi}{\rho_n}\right)^{\beta}(x)\left(\int k(x,y)\nabla\log(\frac{\rho_n}{\pi})(y)d\rho_n(y)\right)^{\top} }_F$. It is also relatively weak, and achievable for example when the potential function of $\rho_n$ does not fluctuate wildly.

\begin{assumption}\label{asp:5}
	$\norm{\nabla \left(\frac{\pi}{\rho_n}\right)^{\beta}}\leq \cfive$ in the region 
	$\left\{x: \left(\frac{\pi}{\rho_n}\right)^{\beta}(x)\leq M_{\rho_n}(\delta)\right\}$.
\end{assumption}

Though Assumptions~\ref{asp:3}, \ref{asp:1} and \ref{asp:5} are relatively reasonable, as we stated, we do not know how to estimate constants $C$, $M_{\rho_n}(\delta)$ and $C_{\rho_n}(\delta)$ beforehand.

With all this preparation, we can now formulate our descent lemma for the population limit \algname{$\beta$-SVGD} when $\beta\in (-1,0)$. The proof can be found in \Cref{apdx:lemmas}. 
\begin{proposition}[Descent Lemma]\label{thm:main1}
	Suppose $\beta\in (-1,0)$, $I_{Stein}\left(\rho_n\mid\pi\right)\geq 2\delta$ and Assumptions \ref{asp:4}, \ref{asp:2}, \ref{asp:1} and \ref{asp:5} hold. Choosing 
	\begin{equation}
		\begin{cases}
			0< \gamma\leq \frac{1}{6\left(\cfive+M_{\rho_n}(\delta)\right)BI_{Stein}\left(\rho_n\mid\pi\right)^{\frac{1}{2}}}\\
			0<\gamma \leq \frac{2(\beta+1)\left(I_{Stein}\left(\rho_n\mid\pi\right)-\delta\right)}{B^2I_{Stein}\left(\rho_n\mid\pi\right)\left(LM_{\rho_n}(\delta)^2+10(\cfive+M_{\rho_n}(\delta))^2\right)}\\
			0<\gamma\leq\frac{\beta+1}{B^2\left(LM_{\rho_n}\left(\delta\right)^2+10(\cfive+M_{\rho_n}\left(\delta\right))^2\right)}
		\end{cases},
	\end{equation}
we have the descent property
	\begin{equation}\label{eq:descentlemma}
		\begin{aligned}
			\squeeze		e^{\beta\mathrm{D}_{\beta+1}\left(\rho_{n+1}\mid\pi\right)}-e^{\beta\mathrm{D}_{\beta+1}\left(\rho_{n}\mid\pi\right)}
			\geq -\beta(\beta+1)\gamma\left(\frac{1}{2}I_{Stein}\left(\rho_n\mid\pi\right)-\delta\right).
		\end{aligned}
	\end{equation}
\end{proposition}

\Cref{thm:main1} contains the descent lemma for the population limit \algname{SVGD}~\cite{liu2017stein}. Actually, let $\beta$ and $\delta$ approach to $0$, the descent lemma for the population limit \algname{SVGD} will be derived by L'Hospital rule. When $\beta>0$, we also have \Cref{eq:descentlemma}, however due to the sign change of $-\beta$, \Cref{eq:descentlemma} can not guarantee $\mathrm{D}_{\beta+1}\left(\rho_{n+1}\mid\pi\right)<\mathrm{D}_{\beta+1}\left(\rho_{n}\mid\pi\right)$ anymore~(for an asymptotic analysis, please refer to \Cref{apdx:mis}). 
\begin{remark}
	The lack of a descent lemma for \algname{$\beta$-SVGD} when $\beta>0$ is not a great loss for us, as explained in \Cref{sec:3.1}, negative $\beta$ is preferable in the implementation of \algname{$\beta$-SVGD}. One can see from our experiments that \algname{$\beta$-SVGD} with negative $\beta$ performs much better than the one with positive $\beta$, this verifies our theory in \Cref{sec:3.1}.
\end{remark} 
The next corollary is a discrete time version of \Cref{thm:betaflow}. Letting $M_{\rho_n}(\varepsilon)$ and $C_{\rho_n}(\varepsilon)$ have consistent upper bound is reasonable since intuitively $\rho_n$ will approach $\pi$, though we can not verify this beforehand.
\begin{corollary}\label{cor:descentlemmacor1}In \Cref{thm:main1}, choose $\delta=\varepsilon$ and suppose Assumptions \ref{asp:4}, \ref{asp:2}, \ref{asp:3}, \ref{asp:1} and \ref{asp:5} hold with uniformly bounded $M_{\rho_n}\left(\varepsilon\right)$ and $C_{\rho_n}\left(\varepsilon\right)$, so that $\gamma$ is uniformly lower bounded. Then we have at most 
	\begin{equation}
		\squeeze	N=\Omega\left(-\frac{2}{\beta(\beta+1)\varepsilon\gamma}\right)
	\end{equation}
	iterations to achieve $\min_{i\in\{0,1,\ldots,N\}}I_{Stein}\left(\rho_i\mid\pi\right)\leq 3\varepsilon$.
\end{corollary}
{\begin{remark}
	We do not claim here that the complexity of \algname{$\beta$-SVGD} is independent of $\pi$ and $\rho_0$, since the upper bound for constant $M_{\rho_n}(\delta)$ and $C_{\rho_n}(\delta)$ are not determined.
\end{remark}}

\section{Conclusion}
We construct a family of continuous time flows called \algname{$\beta$-SVGD} flows on the space of probability distributions, when $\beta\in (-1,0)$, its convergence rate is independent of the initial distribution and the target distribution. Based on \algname{$\beta$-SVGD} flow, we design a family of weighted \algname{SVGD} called \algname{$\beta$-SVGD}. \algname{$\beta$-SVGD} has the similar computation complexity as \algname{SVGD}, and due to the Stein importance weight, it converges more quickly and is more stable than \algname{SVGD} in our experiments. 

{We use importance weight as a preconditioner in the update of \algname{SVGD}, and this idea can be applied to other kinds of sampling algorithms, such as Langevin algorithm. There have been a number of generalised Langevin type dynamics proposed, see \citet{garbuno2020interacting,li2019hessian}, however, the advantages of these dynamics over the original Langevin dynamics are unclear.  Inspired by \algname{$\beta$-SVGD} flow~(\ref{def:betaflow}) and \Cref{thm:betaflow}, we can easily prove a  similar theorem for the importance weighted Langevin dynamic with a stronger Fisher information criterion. We left this for future study.}

\clearpage
\bibliography{iclr2023_conference}
\bibliographystyle{iclr2023_conference}

\appendix
\appendix 
\section{Calculus}\label{apdx:calculus}
This section is devoted to provide rigorous verification for several claims in the main paper, these results are already known to readers who are familiar with R\'enyi divergence.
We first calculate the Wasserstein gradient flow of R\'enyi divergence. Let $\rho_t$ satisfies 
\begin{align*}
	\frac{\partial\rho_t}{\partial t}+\operatorname{div}\left(\rho_tv_t\right)=0, 
\end{align*}
for some vector field $v_t$ on $\R^d$, then when $\alpha\in (0,1)\cup(1,\infty)$, we have 
\begin{align*}
	\frac{d}{dt}\Renyi\left(\rho_t\mid\pi\right)&=\frac{d}{dt}\frac{1}{\alpha-1}\log\left(\int\left(\frac{\rho_t}{\pi}\right)^{\alpha-1}(x)d\rho_t(x)\right)\\
	&=\frac{1}{\alpha-1}\frac{\int\frac{d}{dt}\left(\frac{\rho_t}{\pi}\right)^{\alpha}(x)d\pi(x)}{\int\left(\frac{\rho_t}{\pi}\right)^{\alpha-1}(x)d\rho_t(x)}\\
	&=\frac{\alpha}{\alpha-1}\frac{\int\left(\frac{\rho_t}{\pi}\right)^{\alpha-1}(x)\frac{\partial\rho_t}{\partial t}(x)dx}{\int\left(\frac{\rho_t}{\pi}\right)^{\alpha-1}(x)d\rho_t(x)}\\
	&=-\frac{\alpha}{\alpha-1}\frac{\int\left(\frac{\rho_t}{\pi}\right)^{\alpha-1}(x)\operatorname{div}\left(\rho_tv_t\right)(x)dx}{\int\left(\frac{\rho_t}{\pi}\right)^{\alpha-1}(x)d\rho_t(x)}\\
	&=\frac{\alpha}{\alpha-1}\frac{\int\inner{\nabla\left(\frac{\rho_t}{\pi}\right)^{\alpha-1}(x)}{v_t(x)}d\rho_t(x)}{\int\left(\frac{\rho_t}{\pi}\right)^{\alpha-1}(x)d\rho_t(x)}\\
	&=\alpha\frac{\int\inner{\left(\frac{\rho_t}{\pi}\right)^{\alpha-1}(x)\nabla\log\left(\frac{\rho_t}{\pi}\right)(x)}{v_t(x)}d\rho_t(x)}{\int\left(\frac{\rho_t}{\pi}\right)^{\alpha-1}(x)d\rho_t(x)}\\
	&=\inner{\frac{\alpha\left(\frac{\rho_t}{\pi}\right)^{\alpha-1}\nabla\log(\frac{\rho_t}{\pi})}{\int\left(\frac{\rho_t}{\pi}\right)^{\alpha-1}d\rho_t}}{v_t}_{\rho_t}.
\end{align*}
When $\alpha=1$, we have
\begin{align*}
	\frac{d}{dt}\KL\left(\rho_t\mid\pi\right)&=\frac{d}{dt}\int\log\left(\frac{\rho_t}{\pi}\right)(x)d\rho_t(x)\\
	&=\int\frac{d}{dt}\left\{\frac{\rho_t}{\pi}(x)\log\left(\frac{\rho_t}{\pi}\right)(x)\right\}d\pi(x)\\
	&=\int \left(1+\log\left(\frac{\rho_t}{\pi}\right)(x)\right)\frac{\partial\rho_t}{\partial t}(x)dx\\
	&=-\int \left(1+\log\left(\frac{\rho_t}{\pi}\right)(x)\right)\operatorname{div}\left(\rho_tv_t\right)(x)dx\\
	&=\int\inner{\nabla\log\left(\frac{\rho_t}{\pi}\right)(x)}{v_t(x)}d\rho_t(x)\\
	&=\inner{\nabla\log\left(\frac{\rho_t}{\pi}\right)}{v_t}_{\rho_t}.
\end{align*}
The Wasserstein gradient of the reverse KL-divergence:
\begin{equation*}
	\begin{aligned}
		\frac{d}{dt}\KL\left(\pi\mid\rho_t\right)&:=\frac{d}{dt}\int\log(\frac{\pi}{\rho_t})(x)d\pi(x)\\
		&=-\int \frac{\frac{\partial\rho_t}{\partial t}}{\rho_t}(x)d\pi(x)\\
		&=\int \operatorname{div}\left(\rho_tv_t\right)(x)\frac{\pi}{\rho_t}(x)dx\\
		&=\int \inner{-\nabla\frac{\pi}{\rho_t}(x)}{v_t(x)}d\rho_t(x)\\
		&=\inner{-\nabla\frac{\pi}{\rho_t}}{v_t}_{\rho_t},
	\end{aligned}
\end{equation*}
so it is $-\nabla\frac{\pi}{\rho_t}$.

Next, we verify that $\Renyi\left(\rho\mid\pi\right)\geq 0$.
For $\alpha>1$, we have
\begin{equation*}
	\int\left(\frac{\rho}{\pi}\right)^{\alpha-1}(x)d\rho(x)=\int\left(\frac{\rho}{\pi}\right)^{\alpha}(x)d\pi(x)\geq\left(\int\frac{\rho}{\pi}(x)d\pi(x)\right)^{\alpha}=1
\end{equation*}
by the convexity of function $t^{\alpha}$ for $t\geq 0$, so 
\begin{equation*}
	\Renyi\left(\rho\mid\pi\right)=\frac{1}{\alpha-1}\log\left(\int\left(\frac{\rho}{\pi}\right)^{\alpha-1}(x)d\rho(x)\right)\geq 0.
\end{equation*}
When $\alpha=1$, by the convexity of function $t\log(t)$ for $t\geq0$, we also have 
\begin{equation*}
	\KL\left(\rho\mid\pi\right)=\int\log\left(\frac{\rho_t}{\pi}\right)(x)d\rho(x)
	=\int\frac{\rho_t}{\pi}(x)\log\left(\frac{\rho_t}{\pi}\right)(x)d\pi(x)\geq 0.
\end{equation*}
When $\alpha\in (0,1)$, function $t^{\alpha}$ for $t\geq 0$ is concave, so we first have
\begin{equation*}
	\int\left(\frac{\rho}{\pi}\right)^{\alpha-1}(x)d\rho(x)=\int\left(\frac{\rho}{\pi}\right)^{\alpha}(x)d\pi(x)\leq\left(\int\frac{\rho}{\pi}(x)d\pi(x)\right)^{\alpha}=1,
\end{equation*}
finally
\begin{equation*}
	\Renyi\left(\rho\mid\pi\right)=\frac{1}{\alpha-1}\log\left(\int\left(\frac{\rho}{\pi}\right)^{\alpha-1}(x)d\rho(x)\right)\geq 0.
\end{equation*}
\section{Missing Proofs}\label{apdx:lemmas}

\begin{proof}[\textbf{proof of \Cref{thm:betaflow}}]A direct calculation yields
	\begin{equation}\label{eq:flowprove}
		\begin{aligned}
			\frac{d}{dt}\mathrm{D}_{\beta+1}\left(\rho_t\mid\pi\right)&=\inner{\frac{(\beta+1)\left(\frac{\rho_t}{\pi}\right)^{\beta}\nabla\log(\frac{\rho_t}{\pi})}{\int\left(\frac{\rho_t}{\pi}\right)^{\beta}d\rho_t}}{v_t^{\beta}}_{\rho_t}\quad // \text{refer to \Cref{apdx:calculus} for more calculation details}\\
			&=-\frac{\beta+1}{\int\left(\frac{\rho_t}{\pi}\right)^{\beta}d\rho_t}\iint k(x,y)\inner{\nabla\log(\frac{\rho_t}{\pi})(x)}{\nabla\log(\frac{\rho_t}{\pi})(y)}\left(\frac{\rho_t}{\pi}\right)^{\beta}\left(\frac{\pi}{\rho_t}\right)^{\beta}d\rho_t(x)d\rho_t(y)\\
			&=-(\beta+1)\frac{\iint k(x,y)\inner{\nabla\log(\frac{\rho_t}{\pi})(x)}{\nabla\log(\frac{\rho_t}{\pi})(y)}d\rho_{t}(x)d\rho_{t}(y)}{\int\left(\frac{\rho_t}{\pi}\right)^{\beta}d\rho_t}\leq 0,
		\end{aligned}
	\end{equation}
	which is equivalent to
	\begin{equation}\label{eq:ooo}
		\frac{d}{dt}e^{\beta\mathrm{D}_{\beta+1}\left(\rho_t\mid\pi\right)}= -\beta(\beta+1)I_{Stein}\left(\rho_t\mid\pi\right).
	\end{equation}
	Integrate the above equation for $t$ from $0$ to $T$,  after rearrangement then we will have
	\begin{equation*}
		\begin{aligned}
			\min_{t\in [0,T]}I_{Stein}\left(\rho_{t}\mid\pi\right)&\leq\frac{1}{T}\int_{0}^TI_{Stein}(\rho_t\mid\pi)dt\\ &\leq\frac{\left|e^{\beta\mathrm{D}_{\beta+1}\left(\rho_0\mid\pi\right)}-e^{\beta\mathrm{D}_{\beta+1}\left(\rho_T\mid\pi\right)}\right|}{T|\beta(\beta+1)|}.
		\end{aligned}
	\end{equation*}
	By \myeqref{eq:flowprove}, we know $\mathrm{D}_{\beta+1}\left(\rho_t\mid\pi\right)$ decreases along \algname{$\beta$-SVGD} flow for any $\beta\in (-1,\infty)$. For $\beta>0$, we have
	\begin{equation*}
		\frac{\left|e^{\beta\mathrm{D}_{\beta+1}\left(\rho_0\mid\pi\right)}-e^{\beta\mathrm{D}_{\beta+1}\left(\rho_T\mid\pi\right)}\right|}{T|\beta(\beta+1)|}\leq \frac{e^{\beta\mathrm{D}_{\beta+1}\left(\rho_0\mid\pi\right)}}{T\beta(\beta+1)}.
	\end{equation*}
	For $\beta=0$, we use L'Hospital rule and get 
	\begin{equation*}
		\lim_{\beta\to 0}	\frac{\left|e^{\beta\mathrm{D}_{\beta+1}\left(\rho_0\mid\pi\right)}-e^{\beta\mathrm{D}_{\beta+1}\left(\rho_T\mid\pi\right)}\right|}{T|\beta(\beta+1)|}=\frac{\KL(\rho_0\mid\pi)-\KL(\rho_T\mid\pi)}{T}\leq\frac{\KL(\rho_0\mid\pi)}{T}.
	\end{equation*}
	For $\beta\in (-1,0)$, we have $0\leq e^{\beta\mathrm{D}_{\beta+1}\left(\rho_0\mid\pi\right)}\leq e^{\beta\mathrm{D}_{\beta+1}\left(\rho_T\mid\pi\right)}\leq1$, so $\left|e^{\beta\mathrm{D}_{\beta+1}\left(\rho_0\mid\pi\right)}\leq e^{\beta\mathrm{D}_{\beta+1}\left(\rho_T\mid\pi\right)}\right|\leq 1$ and 
	\begin{equation*}
		\frac{\left|e^{\beta\mathrm{D}_{\beta+1}\left(\rho_0\mid\pi\right)}-e^{\beta\mathrm{D}_{\beta+1}\left(\rho_T\mid\pi\right)}\right|}{T|\beta(\beta+1)|}\leq -\frac{1}{T\beta(\beta+1)}.
	\end{equation*}
	Combine all the three cases, we finish the proof.
\end{proof}
\begin{proof}[\textbf{proof of \Cref{rmk:111}}]
	A similar calculation yields
	\begin{equation*}
		\frac{d}{dt}\int\left(\frac{\rho_t}{\pi}\right)^{\beta+1}(x)d\pi(x)=
		\frac{d}{dt}\int \left(\frac{\pi}{\rho_t}\right)^{-\beta}(x)d\rho_t(x)=-\beta(\beta+1)I_{Stein}\left(\rho_t\mid\pi\right)\leq 0,\quad\text{with} \beta<-1.
	\end{equation*}
	A rearrangement yields
	\begin{equation*}
		\begin{aligned}
			\min_{t\in [0,T]}I_{Stein}\left(\rho_{t}\mid\pi\right)\leq\frac{1}{T}\int_{0}^TI_{Stein}(\rho_t\mid\pi)dt&\leq\frac{\int \left(\frac{\pi}{\rho_0}\right)^{-\beta}(x)d\rho_0(x)-\int \left(\frac{\pi}{\rho_T}\right)^{-\beta}(x)d\rho_T(x)}{|T\beta(\beta+1)|}\\
			&\leq\frac{\int \left(\frac{\pi}{\rho_0}\right)^{-\beta}(x)d\rho_0(x)}{|T\beta(\beta+1)|}\\
			&=\frac{e^{(-\beta-1)\mathrm{D}_{-\beta}\left(\pi\mid\rho_0\right)}}{|T\beta(\beta+1)|},
		\end{aligned}
	\end{equation*}
\end{proof}

\begin{proof}[\textbf{proof of \Cref{lem:1}}]
	Let $g$ be bounded and $\int gd\pi=0$. Let $\epsilon$ be small enough such that $1+\epsilon g\geq 0$, so $\rho:=\pi(1+\epsilon g)$ is a probability distribution and $\rho\ll\pi$. We need first calculate $\KL(\rho\mid\pi)$.
	\begin{equation}
		\begin{aligned}
			\KL(\rho\mid\pi)&=\int\log(\frac{(1+\epsilon g)\pi}{\pi})(x)(1+\epsilon g)(x)d\pi(x)\\
			&=\int (1+\epsilon g)(x)\log(1+\epsilon g)(x)d\pi(x)\\
			&=\int (1+\epsilon g)(x)\left(\epsilon g(x)-\frac{1}{2}\epsilon^2|g|^2(x)\right)d\pi(x)+o(\epsilon^2)\\
			&=\frac{1}{2}\epsilon^2\int|g|^2(x)d\pi(x)+o(\epsilon^2),
		\end{aligned}
	\end{equation}
	in the last step, we used $\int gd\pi=0$. Now we calculate the right hand side of \ref{eq:sli},
	\begin{equation}
		\begin{aligned}
			I_{Stein}(\rho\mid\pi)&=\iint k(x,y)\inner{\nabla\log(\frac{\rho}{\pi})(x)}{\nabla\log(\frac{\rho}{\pi})(y)}d\rho(x)d\rho(y)\\
			&=\iint k(x,y)\inner{\nabla\frac{\rho}{\pi}(x)}{\nabla\frac{\rho}{\pi}(y)}d\pi(x)d\pi(y)\\
			&=\iint k(x,y)\inner{\nabla(1+\epsilon g)(x)}{\nabla(1+\epsilon g)(y)}d\pi(x)d\pi(y)\\
			&=\epsilon^2\iint k(x,y)\inner{\nabla g(x)}{\nabla g(y)}d\pi(x)d\pi(y).
		\end{aligned}
	\end{equation}
	Since we have \Cref{eq:sli}, so
	\begin{equation}
		\frac{1}{2}\epsilon^2\int|g|^2(x)d\pi(x)+o(\epsilon^2)\leq\frac{1}{2\lambda}\epsilon^2\iint k(x,y)\inner{\nabla g(x)}{\nabla g(y)}d\pi(x)d\pi(y),
	\end{equation}
	divide both side by $\epsilon^2$ and let $\epsilon\to 0$, we  have Stein Poincar{\'e} inequality
	\begin{equation}
		\int|g|^{2} d \pi\leq	\frac{1}{\lambda}\iint k(x, y)\inner{\nabla g(x)}{\nabla g(y)}  d \pi(x) d \pi(y).
	\end{equation}
	For general unbounded function $g$ with $\int gd\pi=0$, we can use bounded sequence to approximate it and will also have Stein Poincar{\'e} inequality \ref{eq:spi}
\end{proof}
\begin{proof}[\textbf{proof of \Cref{thm:flowone}}]
	Denoting ${\epsilon_t}^{2}=\int\left(\frac{\rho_{t}-\pi}{\pi}\right)^{2} d \pi$, $f_t=\frac{\rho_{t}-\pi}{{\epsilon_t}}$, then $\int f_t d x=0$, $\int \frac{f_t^{2}}{\pi} d x=1$, 
	$C_{t}:=\int\left(\frac{\rho_{t}}{\pi}\right)^{2} d \pi=1+{\epsilon_t}^{2}$. Thus

	\begin{align*}
		-\frac{d}{d t} \Renyii\left(\rho_{t} \mid \pi\right) &=2 \left<\nabla \log \left(\frac{\rho_{t}}{\pi}\right), v_{t} \right>_{C_{t}^{-1}\left(\frac{\rho_{t}}{\pi}\right)^{2} \pi} \\
		&=\frac{2}{1+{\epsilon_t}^{2}} \iint \inner{\nabla \log \left(\frac{\rho_{t}}{\pi}\right)(y)}{\nabla \log \left(\frac{\rho_{t}}{\pi}\right)(x)}\left(\frac{\rho_{t}}{\pi}\right)^{-1}(y)\left(\frac{\rho_{t}}{\pi}\right)^{2}(y) k(x, y) \left(\frac{\rho_{t}}{\pi}\right)(x) d \pi(x) \pi(y) \\
		&=\frac{2}{1+{\epsilon_t}^{2}} \iint k(x, y)\inner{\nabla\left(\frac{\rho_{t}}{\pi}\right)(x) }{\nabla\left(\frac{\rho_{t}}{\pi}\right)(y)}  d \pi(x) d \pi(y) \\
		&=\frac{2}{1+{\epsilon_t}^{2}} \iint k(x, y)\inner{\nabla\left(\frac{\rho_{t}}{\pi}-1\right)(x) }{\nabla\left(\frac{\rho_{t}}{\pi}-1\right)(y)}  d \pi(x) d \pi(y) \\
		&=\frac{2 {\epsilon_t}^{2}}{1+{\epsilon_t}^{2}} \iint k(x, y)\inner{\nabla\left(\frac{f_{t}}{\pi}\right)(x) }{\nabla\left(\frac{f_{t}}{\pi}\right)(y)}  d \pi(x) d \pi(y).
	\end{align*}
	By Stein Poincar{\'e} inequality, we have 
	\begin{equation*}
		- \iint k(x, y)\inner{\nabla\left(\frac{f_{t}}{\pi}\right)(x) }{\nabla\left(\frac{f_{t}}{\pi}\right)(y)}  d \pi(x) d \pi(y)\leq -\lambda \int \left|\frac{f_t}{\pi}\right|^2(x)d\pi(x),
	\end{equation*}
	so finally we have
	\begin{equation*}
		\begin{aligned}
			\frac{d \Renyii\left(\rho_{t} \mid \pi\right)}{d t}&=-\frac{2 {\epsilon_t}^{2}}{1+{\epsilon_t}^{2}} \iint k(x, y)\inner{\nabla\left(\frac{f_{t}}{\pi}\right)(x) }{\nabla\left(\frac{f_{t}}{\pi}\right)(y)}  d \pi(x) d \pi(y)\\
			&\leq-\frac{2 {\epsilon_t}^{2}}{1+{\epsilon_t}^{2}} \lambda \int\left|\frac{f_t}{\pi}\right|^{2}(x) d \pi(x)\\
			&=-\frac{2 \lambda {\epsilon_t}^{2}}{1+{\epsilon_t}^{2}}\\
			&=-2\lambda\frac{e^{\Renyii(\rho_{t}\mid\pi)}-1}{e^{\Renyii(\rho_{t}\mid\pi)}}\\
			&=-2 \lambda\left(1-e^{-\Renyii\left(\rho_{t} \mid \pi\right)}\right),
		\end{aligned}
	\end{equation*}
	which is equivalent to
	\begin{equation*}
		\frac{d}{dt}\log(e^{\Renyii\left(\rho_t\mid\pi\right)}-1)\leq -2\lambda.
	\end{equation*}
	So
	\begin{equation*}
		\begin{aligned}
			\Renyii\left(\rho_{t} \mid \pi\right) &\leq \log\left(1+\left(e^{\Renyii\left(\rho_0\mid\pi\right)}-1\right)e^{-2\lambda t}\right)\\
			&\leq\left(e^{\Renyii\left(\rho_0\mid\pi\right)}-1\right)e^{-2\lambda t}\\
			&=\frac{e^{\Renyii\left(\rho_0\mid\pi\right)}-1}{\Renyii\left(\rho_0\mid\pi\right)}\Renyii\left(\rho_0\mid\pi\right)e^{-2\lambda t}.
		\end{aligned}
	\end{equation*}
\end{proof}
\begin{proof}[\textbf{proof of \Cref{cor:renyicor}}]
	By \myeqref{eq:exponential2}, when $\alpha\in (0,2)$ we have
	\begin{equation}
		\begin{aligned}
			\Renyi\left(\rho_t\mid\pi\right)&\leq\Renyii\left(\rho_t\mid\pi\right)\\
			&\leq\frac{e^{\Renyii\left(\rho_{0} \mid \pi\right)}-1}{\Renyii\left(\rho_{0} \mid \pi\right)}\Renyii\left(\rho_{0} \mid \pi\right) e^{-2 \lambda t}\\
			&= \frac{e^{\Renyii\left(\rho_{0} \mid \pi\right)}-1}{\Renyii\left(\rho_{0} \mid \pi\right)}\frac{\Renyii\left(\rho_0\mid\pi\right)}{\Renyi\left(\rho_0\mid\pi\right)}{\Renyi\left(\rho_0\mid\pi\right)}e^{-2\lambda t}\\
			&=\frac{e^{\Renyii\left(\rho_{0} \mid \pi\right)}-1}{\Renyi\left(\rho_{0} \mid \pi\right)}{\Renyi\left(\rho_0\mid\pi\right)}e^{-2\lambda t}.
		\end{aligned}
	\end{equation}
\end{proof}
\begin{proof}[\textbf{proof of \Cref{thm:main1}}]\footnotesize
	For simplicity, we will denote $M:=M_{\rho_n}(\delta)$. Define $g_n(x)={\left(\frac{\pi}{\rho_n}\right)^{\beta}(x)\wedge M\int k(x,y)\nabla\log\left(\frac{\rho_n}{\pi}\right)(y)d\rho_n(y)}$,
	$\phi_n(x):=x-\gamma g_n(x)$ and $\rho_{n+1}={\phi_n}_{\#}\rho_n$.
	Then we have
	\begin{equation}
		\begin{aligned}
			e^{\beta\mathrm{D}_{\beta+1}\left(\rho_{n+1}\mid\pi\right)}-e^{\beta\mathrm{D}_{\beta+1}\left(\rho_{n}\mid\pi\right)}&=	e^{\beta\mathrm{D}_{\beta+1}\left(\rho_{n}\mid{\phi_n^{-1}}_{\#}\pi\right)}-e^{\beta\mathrm{D}_{\beta+1}\left(\rho_{n}\mid\pi\right)}\\
			&=\int\left(\frac{\rho_n}{{\phi_n^{-1}}_{\#}\pi}\right)^{\beta}(x)d\rho_n(x)-\int\left(\frac{\rho_n}{\pi}\right)^{\beta}(x)d\rho_n(x)\\
			&=\int\left(\frac{\rho_n}{\pi}\right)^{\beta}(x)\left(\left(\frac{\pi(x)}{{\phi_n^{-1}}_{\#}\pi(x)}\right)^{\beta}-1\right)d\rho_n(x).
		\end{aligned}
	\end{equation}
	We need to upper bound term $I$ and term $II$ in the next equation,
	\begin{equation}
		\left(\frac{\pi(x)}{{\phi_n^{-1}}_{\#}\pi(x)}\right)^{\beta}=\left(\frac{\pi(x)}{\pi(\phi_n(x))\left|\det\operatorname{D}\phi_n\right|(x)}\right)^{\beta}=\exp\left(	\beta\Big(\underbrace{\log(\pi)(x)-\log(\pi)(\phi_n(x))}_{I}\underbrace{-\log(\left|\det\operatorname{D}\phi_n\right|)(x)}_{II}\Big)\right).
	\end{equation}
	For term $I$, we have that 
	\begin{equation}
		\begin{aligned}
			I&=\log(\pi)(x)-\log(\pi)(\phi_n(x))\\
			&=V(x)-V(x-\gamma g_n(x))\\
			&=\gamma\inner{\nabla V(x)}{g_n(x)}-\int_{0}^{\gamma}(t-\gamma)\inner{g_n(x)}{\nabla^2V(x-tg_n(x))g_n(x)}dt\\
			&\leq \gamma\inner{\nabla V(x)}{g_n(x)}-L\int_{0}^{\gamma}(t-\gamma)\normsq{g_n(x)}dt\\
			&=\gamma\inner{\nabla V(x)}{g_n(x)}+\frac{L\gamma^2}{2}\normsq{g_n(x)}.
		\end{aligned}
	\end{equation}
	For term $II$, we have by \Cref{lem:2} that if $\gamma$ satisfies $0 \leq \gamma<\frac{1}{6\left(\cfive+M\right)BI_{Stein}\left(\rho_n\mid\pi\right)^{\frac{1}{2}}}\leq\frac{1}{6\max_{x\in\R^d}\norm{B(x)}_F}$~(see \Cref{eq:uiui}) with $B(x)=\nabla g_n(x)$, then
	\begin{equation}
		II\leq \gamma\operatorname{div}\left(g_n(x)\right)+5\gamma^2\norm{\nabla g_n(x)}_{F}^2.
	\end{equation}
	So all in all, we have 
	\begin{equation}
		\beta\left(I+II\right)\geq \beta\gamma \left(\inner{\nabla V(x)}{g_n(x)}+\operatorname{div}\left(g_n(x)\right)+\gamma\left(\frac{L}{2}\normsq{g_n(x)}+5\normsq{\nabla g_n(x)}_{F}\right)\right).
	\end{equation}
	
	We apply Jensen inequality $\psi\left(\E\left[f(X)\right]\right)\leq\E\left[\psi\left(f(X)\right)\right]$ with $\psi(x)=e^{x}-1$ convex and $f(x)=	\beta\Big(\log(\pi)(x)-\log(\pi)(\phi_n(x))-\log(\left|\det\operatorname{D}\phi_n\right|)(x)\Big)$, then we have when $\beta\in(-1,0)$ that
	\begin{equation}\hspace{0cm}
		\label{eq:tttt}
		\begin{aligned}
			&e^{\beta\mathrm{D}_{\beta+1}\left(\rho_{n+1}\mid\pi\right)}-e^{\beta\mathrm{D}_{\beta+1}\left(\rho_{n}\mid\pi\right)}\\
			&=\left(\int \left(\frac{\rho_n}{\pi}\right)^{\beta}(x)\rho_n(x)dx\right)\frac{\int\left(\left(\frac{\pi(x)}{{\phi_n^{-1}}_{\#}\pi(x)}\right)^{\beta}-1\right)\left(\frac{\rho_n}{\pi}\right)^{\beta}(x)\rho_n(x)dx}{\int \left(\frac{\rho_n}{\pi}\right)^{\beta}(x)\rho_n(x)dx}\\
			&=\left(\int \left(\frac{\rho_n}{\pi}\right)^{\beta}(x)\rho_n(x)dx\right)\frac{\int\left(\exp\left(	\beta\Big(\log(\pi)(x)-\log(\pi)(\phi_n(x))-\log(\left|\det\operatorname{D}\phi_n\right|)(x)\Big)\right)-1\right)\left(\frac{\rho_n}{\pi}\right)^{\beta}(x)\rho_n(x)dx}{\int \left(\frac{\rho_n}{\pi}\right)^{\beta}(x)\rho_n(x)dx}\\
			&\geq \left(\int \left(\frac{\rho_n}{\pi}\right)^{\beta}(x)\rho_n(x)dx\right)\left\{\exp\left(\frac{\int \beta\Big(\log(\pi)(x)-\log(\pi)(\phi_n(x))-\log(\left|\det\operatorname{D}\phi_n\right|)(x)\Big)\left(\frac{\rho_n}{\pi}\right)^{\beta}(x)d\rho_n(x)}{\int \left(\frac{\rho_n}{\pi}\right)^{\beta}(x)\rho_n(x)dx}\right)-1\right\}\\
			&\geq \left(\int \left(\frac{\rho_n}{\pi}\right)^{\beta}(x)\rho_n(x)dx\right)\exp\left(\frac{\int \beta\gamma \left(\inner{\nabla V(x)}{g_n(x)}+\operatorname{div}\left(g_n(x)\right)+\gamma\left(\frac{L}{2}\normsq{g_n(x)}+5\normsq{\nabla g_n(x)}_{F}\right)\right)\left(\frac{\rho_n}{\pi}\right)^{\beta}(x)d\rho_n(x)}{\int \left(\frac{\rho_n}{\pi}\right)^{\beta}(x)\rho_n(x)dx}\right)\\
			&\quad -\int \left(\frac{\rho_n}{\pi}\right)^{\beta}(x)\rho_n(x)dx\\
			&\geq\left(\int \left(\frac{\rho_n}{\pi}\right)^{\beta}(x)\rho_n(x)dx\right)\\
			&\quad\times \left\{\exp\left(\frac{\beta\gamma\int \left(\frac{\rho_n}{\pi}\right)^{\beta}(x)\Big(\inner{\nabla V(x)}{g_n(x)}+\operatorname{div}\left(g_n(x)\right)\Big)d\rho_n(x)}{\int \left(\frac{\rho_n}{\pi}\right)^{\beta}(x)\rho_n(x)dx}+\beta \gamma^2\max_{x\in\R^d}\left(\frac{L}{2}\normsq{g_n(x)}+5\normsq{\nabla g_n(x)}_{F}\right)\right)-1\right\}.
		\end{aligned}
	\end{equation}
	We need to calculate term $III:=\int \left(\frac{\rho_n}{\pi}\right)^{\beta}(x)\Big(\inner{\nabla V(x)}{g_n(x)}+\operatorname{div}\left(g_n(x)\right)\Big)d\rho_n(x)$, we have
	
	\begin{equation}\label{eq:termthreemore}
		\begin{aligned}
			III&=\int \left(\frac{\rho_n}{\pi}\right)^{\beta}(x)\inner{\nabla V(x)}{\left(\frac{\pi}{\rho_n}\right)^{\beta}(x)\wedge M\int k(x,y)\nabla\log\left(\frac{\rho_n}{\pi}\right)(y)d\rho_n(y)}d\rho_n(x)\\
			&\quad-\int \inner{\nabla\left\{\rho_n(x)\left(\frac{\rho_n}{\pi}\right)^{\beta}(x)\right\}}{\left(\frac{\pi}{\rho_n}\right)^{\beta}(x)\wedge M\int k(x,y)\nabla\log\left(\frac{\rho_n}{\pi}\right)(y)d\rho_n(y)}dx\\
			&=\int \left(\frac{\rho_n}{\pi}\right)^{\beta}(x)\left(\frac{\pi}{\rho_n}\right)^{\beta}(x)\wedge M\inner{\nabla V(x)}{\int k(x,y)\nabla\log\left(\frac{\rho_n}{\pi}\right)(y)d\rho_n(y)}d\rho_n(x)\\
			&\quad-\int \inner{\left(\frac{\rho_n}{\pi}\right)^{\beta}(x)\nabla\rho_n(x)+\beta\rho_n(x)\left(\frac{\rho_n}{\pi}\right)^{\beta}(x)\nabla\log(\frac{\rho_n}{\pi})(x)}{\left(\frac{\pi}{\rho_n}\right)^{\beta}(x)\wedge M\int k(x,y)\nabla\log\left(\frac{\rho_n}{\pi}\right)(y)d\rho_n(y)}dx\\
			&=\int \left(\frac{\rho_n}{\pi}\right)^{\beta}(x)\left(\frac{\pi}{\rho_n}\right)^{\beta}(x)\wedge M\inner{\nabla V(x)}{\int k(x,y)\nabla\log\left(\frac{\rho_n}{\pi}\right)(y)d\rho_n(y)}d\rho_n(x)\\&\quad-\int \left(\frac{\rho_n}{\pi}\right)^{\beta}(x)M^{-1}\vee\left(\frac{\pi}{\rho_n}\right)^{\beta}(x)\wedge M\inner{\nabla \log\left(\rho_n\right)(x)}{\int k(x,y)\nabla\log\left(\frac{\rho_n}{\pi}\right)(y)d\rho_n(y)}d\rho_n(x)
			\\
			&\quad-\beta \iint \left(\frac{\rho_n}{\pi}\right)^{\beta}(x) \left(\frac{\pi}{\rho_n}\right)^{\beta}\wedge Mk(x,y)\inner{\nabla\log\left(\frac{\rho_n}{\pi}\right)(x)}{\nabla\log\left(\frac{\rho_n}{\pi}\right)(y)}d\rho_n(x)d\rho_n(y)\\
			&=-(\beta+1)\iint \left(\frac{\rho_n}{\pi}\right)^{\beta}(x) \left(\frac{\pi}{\rho_n}\right)^{\beta}\wedge Mk(x,y)\inner{\nabla\log\left(\frac{\rho_n}{\pi}\right)(x)}{\nabla\log\left(\frac{\rho_n}{\pi}\right)(y)}d\rho_n(x)d\rho_n(y)\\
			&\leq -(\beta+1)\left(I_{Stein}\left(\rho_n\mid\pi\right)-\delta\right).
		\end{aligned}
	\end{equation}
	So combine \Cref{eq:tttt}, we have
	\begin{equation}
		\begin{aligned}
			&e^{\beta\mathrm{D}_{\beta+1}\left(\rho_{n+1}\mid\pi\right)}-e^{\beta\mathrm{D}_{\beta+1}\left(\rho_{n}\mid\pi\right)}\\&\geq \left(\int \left(\frac{\rho_n}{\pi}\right)^{\beta}(x)\rho_n(x)dx\right) \left\{\exp\left(\frac{-\beta(\beta+1)\gamma \left(I_{Stein}\left(\rho_n\mid\pi\right)-\delta\right)}{\int \left(\frac{\rho_n}{\pi}\right)^{\beta}(x)\rho_n(x)dx}+\beta \gamma^2\max_{x\in\R^d}\left(\frac{L}{2}\normsq{g_n(x)}+5\normsq{\nabla g_n(x)}_{F}\right)\right)-1\right\}.
		\end{aligned}
	\end{equation}
	
	Now, we need to bound $\max_{x\in\R^d}\frac{L}{2}\normsq{g_n(x)}+5\normsq{\nabla g_n(x)}_{F}$. First denote $s(x):=\int k(x,y)\nabla\log\left(\frac{\rho_n}{\pi}\right)(y)d\rho_n(y)$, and we have
	\begin{equation}
		\norm{s(x)}=\sqrt{\sum_{i=1}^d\normsq{s_i(x)}}=\sqrt{\sum_{i=1}^d\normsq{\inner{s_i(\cdot)}{k(x,\cdot)}_{\cH_0}}}\leq \sqrt{\sum_{i=1}^dB^2\norm{s_i}^2_{\cH_0}}=B\norm{s}_{\cH}=B I_{Stein}\left(\rho_n\mid\pi\right)^{\frac{1}{2}}
	\end{equation}
	and
	\begin{equation}
		\begin{aligned}
			\|\nabla s(x)\|_{F} &=\sqrt{ \sum_{i, j=1}^{d}\left|\frac{\partial s_i(x)}{\partial x_{j}}\right|^{2}}\\
			&=\sqrt{\sum_{i, j=1}^{d}\left\langle\partial_{x_{j}} k(x, .), s_i\right\rangle_{\mathcal{H}_{0}}} \\
			&\leq \sqrt{\sum_{i, j=1}^{d}\left\|\partial_{x_{j}} k(x, .)\right\|_{\mathcal{H}_{0}}^{2}\left\|s_{i}\right\|_{\mathcal{H}_{0}}^{2} }\\
			&=\sqrt{\|\nabla k(x, .)\|_{\mathcal{H}}^{2}\left\|s\right\|_{\mathcal{H}}^{2}} \\
			&\leq \sqrt{B^{2}\left\|s\right\|_{\mathcal{H}}^{2} }=B I_{Stein}\left(\rho_n\mid\pi\right)^{\frac{1}{2}}.
		\end{aligned}
	\end{equation}
	Then we have $\norm{g_n(x)}\leq M\norm{s(x)}\leq MBI_{Stein}\left(\rho_n\mid\pi\right)^{\frac{1}{2}}$ and
	\begin{equation}\label{eq:uiui}
		\begin{aligned}
			\norm{\nabla g_n(x)}_{F}&=\norm{\nabla\left(\frac{\pi}{\rho_n}\right)^{\beta}(x)s(x)^{\top}1_{\left(\frac{\pi}{\rho_n}\right)^{\beta}(x)\in [0,M]}(x)+\left(\frac{\pi}{\rho_n}\right)^{\beta}(x)\wedge M\nabla s(x)}_F\\
			&\leq \norm{\nabla\left(\frac{\pi}{\rho_n}\right)^{\beta}(x)s(x)^{\top}1_{\left(\frac{\pi}{\rho_n}\right)^{\beta}(x)\in [0,M]}(x)}_F+\norm{\left(\frac{\pi}{\rho_n}\right)^{\beta}(x)\wedge M\nabla s(x)}_F\\
			&\leq \cfive B I_{Stein}\left(\rho_n\mid\pi\right)^{\frac{1}{2}}+MB I_{Stein}\left(\rho_n\mid\pi\right)^{\frac{1}{2}}=(\cfive+M)B I_{Stein}\left(\rho_n\mid\pi\right)^{\frac{1}{2}}.
		\end{aligned}
	\end{equation}
	So we have 
	\begin{equation}
		\max_{x\in\R^d}\frac{L}{2}\normsq{g_n(x)}+5\normsq{\nabla g_n(x)}_{F}\leq \left(\frac{L}{2}M^2+5(\cfive+M)^2\right)B^2I_{Stein}\left(\rho_n\mid\pi\right)
	\end{equation}
	and 
	\begin{equation}
		\begin{aligned}
			&e^{\beta\mathrm{D}_{\beta+1}\left(\rho_{n+1}\mid\pi\right)}-e^{\beta\mathrm{D}_{\beta+1}\left(\rho_{n}\mid\pi\right)}\\&\geq \left(\int \left(\frac{\rho_n}{\pi}\right)^{\beta}(x)\rho_n(x)dx\right) \exp\left(\frac{-\beta(\beta+1)\gamma \left(I_{Stein}\left(\rho_n\mid\pi\right)-\delta\right)}{\int \left(\frac{\rho_n}{\pi}\right)^{\beta}(x)\rho_n(x)dx}+\beta \gamma^2B^2I_{Stein}\left(\rho_n\mid\pi\right)\left(\frac{L}{2}M^2+5(\cfive+M)^2\right)\right)\\
			&\quad -\int \left(\frac{\rho_n}{\pi}\right)^{\beta}(x)\rho_n(x)dx
		\end{aligned}
	\end{equation}
	
	Since $\int \left(\frac{\rho_n}{\pi}\right)^{\beta}d\rho_n(x)\leq 1$ when $\beta\in (-1,0)$, so set $\gamma\leq \frac{2(\beta+1)\left(I_{Stein}\left(\rho_n\mid\pi\right)-\delta\right)}{B^2I_{Stein}\left(\rho_n\mid\pi\right)\left(LM^2+10(\cfive+M)^2\right)}$, then we have $\frac{-\beta(\beta+1)\gamma \left(I_{Stein}\left(\rho_n\mid\pi\right)-\delta\right)}{\int \left(\frac{\rho_n}{\pi}\right)^{\beta}(x)\rho_n(x)dx}+\beta \gamma^2B^2I_{Stein}\left(\rho_n\mid\pi\right)\left(\frac{L}{2}M^2+5(\cfive+M)^2\right)\geq 0$. Finally we use $e^x\geq 1+x$ when $x\geq 0$ to get
	\begin{equation}
		\begin{aligned}
			e^{\beta\mathrm{D}_{\beta+1}\left(\rho_{n+1}\mid\pi\right)}-e^{\beta\mathrm{D}_{\beta+1}\left(\rho_{n}\mid\pi\right)}
			&\geq -\beta(\beta+1)\left(I_{Stein}\left(\rho_n\mid\pi\right)-\delta\right)\\
			&\quad+\beta \gamma^2B^2I_{Stein}\left(\rho_n\mid\pi\right)\left(\frac{L}{2}M^2+5(\cfive+M)^2\right)e^{\beta\rm D_{\beta+1}\left(\rho_n\mid\pi\right)}\\
			&\geq -\beta(\beta+1)\gamma\left(\frac{1}{2}I_{Stein}\left(\rho_n\mid\pi\right)-\delta\right),
		\end{aligned}
	\end{equation}
	the last line is because we choose $\gamma\leq \frac{\beta+1}{B^2\left(LM^2+10(\cfive+M)^2\right)}$.
\end{proof}

\begin{proof}[\textbf{proof of \Cref{cor:descentlemmacor1}}]
	\footnotesize
	Due to \Cref{thm:main1}, we have
	\begin{equation}\label{eq:betasmall}
		\	e^{\beta\mathrm{D}_{\beta+1}\left(\rho_{n+1}\mid\pi\right)}-e^{\beta\mathrm{D}_{\beta+1}\left(\rho_{n}\mid\pi\right)}
		\geq -\beta(\beta+1)\gamma\left(\frac{1}{2}I_{Stein}\left(\rho_n\mid\pi\right)-\varepsilon\right).
	\end{equation}
	Without loss of generality, we suppose $I_{Stein}\left(\rho_i\mid\pi\right)\geq 2\varepsilon$ for $i=0,1,\ldots,N$. We take summation of \Cref{eq:betasmall} for $i=0,1,\ldots,N$,
	\begin{equation}
		\begin{aligned}
			\min_{i\in\{0,1,\ldots,N\}}(\beta+1)\left(\frac{1}{2}I_{Stein}\left(\rho_i\mid\pi\right)-\varepsilon\right)&\leq\frac{e^{\beta\mathrm{D}_{\beta+1}\left(\rho_{N+1}\mid\pi\right)}-e^{\beta\mathrm{D}_{\beta+1}\left(\rho_{0}\mid\pi\right)}}{-N\beta\gamma}\\
			&\leq-\frac{1}{N\beta\gamma},
		\end{aligned}
	\end{equation}
	so
	\begin{equation}
		\min_{i\in\{0,1,\ldots,N\}}I_{Stein}\left(\rho_i\mid\pi\right)\leq -\frac{2}{N\beta(\beta+1)\gamma}+2\varepsilon,
	\end{equation}
	so when $N\geq-\frac{2}{\beta(\beta+1)\varepsilon\gamma}$ we have
	\begin{equation}
		\min_{i\in\{0,1,\ldots,N\}}I_{Stein}\left(\rho_i\mid\pi\right)\leq -\frac{2}{N\beta(\beta+1)\gamma}+2\varepsilon\leq 3\varepsilon.
	\end{equation}
\end{proof}

\section{Miscellaneous}\label{apdx:mis}

The following proposition is the asymptotic analysis for population limit \algname{$\beta$-SVGD} when $\beta>0$. 
\begin{proposition}\label{thm:main} Suppose $\beta>0$ and $I_{Stein}\left(\rho_n\mid\pi\right)\geq \delta$.
	Let Assumptions \ref{asp:4},\ref{asp:2},\ref{asp:1}  and \ref{asp:5} hold. Suppose \newline
	$\max_{x\in\R^d}\left|\inner{\nabla V(x)}{g_n(x)}+\operatorname{div}\left(g_n(x)\right)\right|\leq \cthree$. Choose 
	\begin{equation}
		0\leq \gamma\ll\min\left\{\frac{1}{\left(C_{\rho_n}(\delta)+M_{\rho_n}(\delta)\right)BI_{Stein}\left(\rho_n\mid\pi\right)^{\frac{1}{2}}},\frac{1}{\cthree}\right\},
	\end{equation}
	then 
	\begin{equation}\label{eq:main}
		\begin{aligned}
			e^{\beta\mathrm{D}_{\beta+1}\left(\rho_{n+1}\mid\pi\right)}-e^{\beta\mathrm{D}_{\beta+1}\left(\rho_{n}\mid\pi\right)}=-\beta(\beta+1)\gamma \left(I_{Stein}\left(\rho_n\mid\pi\right)-\cO(\gamma)e^{\beta\mathrm{D}_{\beta+1}\left(\rho_{n}\mid\pi\right)}\right).
		\end{aligned}
	\end{equation}
	
\end{proposition}
\begin{proof}[\textbf{proof of \Cref{thm:main}}]\footnotesize
	Same as in the proof of \Cref{thm:main1}, we need to estimate term $I$ and $II$ in the following 
	\begin{equation}
		\beta\Big(\underbrace{\log(\pi)(x)-\log(\pi)(\phi_n(x))}_{I}\underbrace{-\log(\left|\det\operatorname{D}\phi_n\right|)(x)}_{II}\Big).
	\end{equation}
	For term $I$, we have
	\begin{equation}\label{eq:termone}
		\begin{aligned}
			I&=\log(\pi)(x)-\log(\pi)(\phi_n(x))\\
			&=V(x)-V(x-\gamma g_n(x))\\
			&=\gamma\inner{\nabla V(x)}{g_n(x)}-\int_{0}^{\gamma}(t-\gamma)\inner{g_n(x)}{\nabla^2V(x-tg_n(x))g_n(x)}dt\\
			&\leq \gamma\inner{\nabla V(x)}{g_n(x)}-\int_{0}^{\gamma}(t-\gamma)\normsq{g_n(x)}\norm{\nabla^2V(x-tg_n(x))}_{op}dt\\
			&\leq \gamma\inner{\nabla V(x)}{g_n(x)}-\int_{0}^{\gamma}(t-\gamma)\normsq{g_n(x)}Ldt\\
			&=\gamma\inner{\nabla V(x)}{g_n(x)}+\frac{L\gamma^2}{2}\normsq{g_n(x)}.
		\end{aligned}
	\end{equation}
	Similarly we have
	\begin{equation}
		I\geq \gamma\inner{\nabla V(x)}{g_n(x)}-\frac{L\gamma^2}{2}\normsq{g_n(x)}.
	\end{equation}
	For term $II$, we need to apply \Cref{lem:2} to matrix $B=-\nabla g_n$, then based on the condition on $\gamma$ 
	we have 
	\begin{equation}
		\begin{aligned}
			-\log(\left|\det\operatorname{D}\phi_n\right|)(\phi_n(x))&\leq \gamma\operatorname{tr}\left(\nabla g_n(x)\right)+5\gamma^2\ctwo^2\\
			&=\gamma\operatorname{div}\left(g_n(x)\right)+5\gamma^2\ctwo^2
		\end{aligned}
	\end{equation}
	and 
	\begin{equation}
		\begin{aligned}
			-\log(\left|\det\operatorname{D}\phi_n\right|)(\phi_n(x))&\geq \gamma\operatorname{tr}\left(\nabla g_n(x)\right)+2\gamma^2\ctwo^2\\
			&=\gamma\operatorname{div}\left(g_n(x)\right)+2\gamma^2\ctwo^2.
		\end{aligned}
	\end{equation}
	So we have 
	\begin{equation}
		\begin{aligned}
			I+II&\leq \gamma\inner{\nabla V(x)}{g_n(x)}+\frac{L\cone^2\gamma^2}{2}+\gamma\operatorname{div}\left(g_n(x)\right)+5\gamma^2\ctwo^2\\
			&=\gamma\left(\inner{\nabla V(x)}{g_n(x)}+\operatorname{div}\left(g_n(x)\right)\right)+\gamma^2\frac{L\cone^2+10\ctwo^2}{2}.
		\end{aligned}
	\end{equation}
	Similarly, we can build 
	\begin{equation}
		I+II\geq \gamma\left(\inner{\nabla V(x)}{g_n(x)}+\operatorname{div}\left(g_n(x)\right)\right)+\gamma^2\frac{L\cone^2+4\ctwo^2}{2}.
	\end{equation}
	So
	\begin{equation}
		\begin{aligned}
			&\left(\frac{\pi(x)}{{\phi_n^{-1}}_{\#}\pi(x)}\right)^{\beta}-1\\&=e^{\beta\left(I+II\right)}-1\\
			&\leq\beta\gamma\left(\inner{\nabla V(x)}{g_n(x)}+\operatorname{div}(g_n(x))\right)+\cO(\gamma^2),
		\end{aligned}
	\end{equation}
	where we use the assumption that $\max_{x\in\R^d}\left|\inner{\nabla V(x)}{g_n(x)}+\operatorname{div}\left(g_n(x)\right)\right|\leq \cthree$ and $\gamma\ll \frac{1}{\max\left\{\cone,\ctwo,\cthree\right\}}$.
	
	Now we arrive at
	\begin{equation}
		\begin{aligned}
			&\int\left(\frac{\rho_n}{\pi}\right)^{\beta}(x)\left(\left(\frac{\pi(x)}{{\phi_n^{-1}}_{\#}\pi(x)}\right)^{\beta}-1\right)d\rho_n(x)\\
			&=\int \left(\frac{\rho_n}{\pi}\right)^{\beta}(x)\Big(\beta\gamma\left(\inner{\nabla V(x)}{g_n(x)}+\operatorname{div}(g_n(x))\right)+\cO(\gamma^2)\Big)d\rho_n(x)\\
			&=\beta\gamma\underbrace{\int \left(\frac{\rho_n}{\pi}\right)^{\beta}(x)\Big(\inner{\nabla V(x)}{g_n(x)}+\operatorname{div}\left(g_n(x)\right)\Big)d\rho_n(x)}_{III}+\cO(\gamma^2)e^{\beta\mathrm{D}_{\beta+1}\left(\rho_n\mid\pi\right)}\\
			&\leq-\beta(\beta+1)\gamma\left(I_{Stein}\left(\rho_n\mid\pi\right)-\delta\right)+\cO(\gamma^2)e^{\beta\mathrm{D}_{\beta+1}\left(\rho_n\mid\pi\right)}.
		\end{aligned}
	\end{equation}
	
	Combine all of these, we finally have
	\begin{equation}
		\begin{aligned}
			e^{\beta\mathrm{D}_{\beta+1}\left(\rho_{n+1}\mid\pi\right)}-e^{\beta\mathrm{D}_{\beta+1}\left(\rho_{n}\mid\pi\right)}\leq -\beta(\beta+1)\gamma\left(I_{Stein}\left(\rho_n\mid\pi\right)-\delta-\cO(\gamma)e^{\beta\mathrm{D}_{\beta+1}\left(\rho_n\mid\pi\right)}\right).
		\end{aligned}
	\end{equation}
\end{proof}
\begin{corollary}\label{cor:descentlemmacor2}
	In \Cref{thm:main1}, choose $\delta=\varepsilon$ and suppose Assumptions \ref{asp:4},\ref{asp:2},\ref{asp:3},\ref{asp:1} and \ref{asp:5} hold with uniformly bounded $M_{\rho_n}\left(\varepsilon\right)$, $C_{\rho_n}\left(\varepsilon\right)$ and $\cthree$.If we further set $\gamma\ll \frac{\varepsilon}{e^{\beta\rm D_{\beta+1}\left(\rho_0\mid\pi\right)}}$, then we need 
	\begin{equation}
		N=
		\Omega\left(\frac{e^{\beta\mathrm{D}_{\beta+1}\left(\rho_0\mid\pi\right)}}{\beta(\beta+1)\varepsilon\gamma}\right)
	\end{equation}
	steps to get $\min_{i\in\{0,1,\ldots,N\}}I_{Stein}\left(\rho_i\mid\pi\right)\leq 2\varepsilon$.
\end{corollary}
\begin{proof}[\textbf{proof of \Cref{cor:descentlemmacor2}}]
	\footnotesize
	Due to \Cref{thm:main}, we have
	\begin{equation}
		\begin{aligned}
			e^{\beta\mathrm{D}_{\beta+1}\left(\rho_{i+1}\mid\pi\right)}-e^{\beta\mathrm{D}_{\beta+1}\left(\rho_{i}\mid\pi\right)}=-\beta(\beta+1)\gamma \left(I_{Stein}\left(\rho_i\mid\pi\right)-\varepsilon-\cO(\gamma)e^{\beta\mathrm{D}_{\beta+1}\left(\rho_{i}\mid\pi\right)}\right).
		\end{aligned}
	\end{equation}
	add from $i=0$ to $i=N$, we have
	\begin{equation}
		\begin{aligned}
			\beta\left(\beta+1\right)\gamma\sum_{i=0}^N\left(I_{Stein}\left(\rho_i\mid\pi\right)-\varepsilon-\cO(\gamma)e^{\beta\mathrm{D}_{\beta+1}\left(\rho_{i}\mid\pi\right)}\right)&=-\sum_{i=0}^N \left({e^{\beta\rm{D}_{\beta+1}\left(\rho_{i+1}\mid\pi\right)}}-{e^{\beta\rm{D}_{\beta+1}\left(\rho_{i}\mid\pi\right)}}\right)\\
			&=e^{\beta\mathrm{D}_{\beta+1}\left(\rho_{0}\mid\pi\right)}-e^{\beta\mathrm{D}_{\beta+1}\left(\rho_{N+1}\mid\pi\right)},
		\end{aligned}
	\end{equation}
	so we finally have
	\begin{equation}
		\min_{i\in\{0,1,\ldots,N\}}\left(I_{Stein}\left(\rho_i\mid\pi\right)-\varepsilon-\cO(\gamma)e^{\beta\mathrm{D}_{\beta+1}\left(\rho_{i}\mid\pi\right)}\right)\leq \frac{e^{\beta\mathrm{D}_{\beta+1}\left(\rho_{0}\mid\pi\right)}-e^{\beta\mathrm{D}_{\beta+1}\left(\rho_{N+1}\mid\pi\right)}}{\beta\left(\beta+1\right)N\gamma}.
	\end{equation}
	For any error bound $\varepsilon$, suppose
	$\min_{i\in\{0,1,\ldots,N\}}I_{Stein}\left(\rho_i\mid\pi\right)\geq2\varepsilon$.For $\beta> 0$,
	we can further require $\gamma\ll \frac{\varepsilon}{e^{\beta\rm D_{\beta+1}\left(\rho_0\mid\pi\right)}}$, then by induction we can easily get $I_{Stein}\left(\rho_i\mid\pi\right)-\varepsilon-\cO(\gamma)e^{\beta\mathrm{D}_{\beta+1}\left(\rho_{i}\mid\pi\right)}\geq 0$ for any $i\in\{0,1,\ldots,N\}$. 
	So all in all, to get $\min_{i\in \{0,1\ldots,N\}}I_{Stein}\left(\rho_i\mid\pi\right)\leq 2\varepsilon$, we need $N=\Omega\left(\frac{e^{\beta\mathrm{D}_{\beta+1}\left(\rho_{0}\mid\pi\right)}-e^{\beta\mathrm{D}_{\beta+1}\left(\rho_{N+1}\mid\pi\right)}}{\beta\left(\beta+1\right)\varepsilon\gamma}\right)$.
\end{proof}
The next lemma is similar to the one from \cite{liu2017stein}, but with both lower and upper bounds for the log determinant term.
\begin{lemma}\label{lem:2}
	Let $B$ be a square matrix and $\|B\|_{F}=\sqrt{\sum_{i j} b_{i j}^{2}}$ its Frobenius norm. Let $\epsilon$ be a positive number that satisfies $0 \leq \gamma<\frac{1}{3\norm{B}_F}$, where $\varrho(\cdot)$ denotes the spectrum radius. Then $I+\epsilon\left(B+B^{\top}\right)+\epsilon^2BB^{\top}$ is positive definite, and
	\begin{equation}
		\begin{aligned}
			&\epsilon \operatorname{tr}(B)-\frac{\epsilon^{2}}{4} \left(\frac{9\|B\|_{F}^{2}}{1-3\epsilon \norm{B}_F}+2{\norm{B}^2_F}\right)\\
			&\leq	\log |\operatorname{det}(I+\epsilon B)| \\
			&\leq \epsilon \operatorname{tr}(B)-\frac{\epsilon^{2}}{4} \left(\frac{9\|B\|_{F}^{2}}{1+3\epsilon \norm{B}_F}+2{\norm{B}^2_F}\right).
		\end{aligned}
	\end{equation}
	Therefore, take an even smaller $\epsilon$ such that $0 \leq \epsilon \leq \frac{1}{6\norm{B}_F}$, we get
	$$
	\epsilon \operatorname{tr}(B)-5 \epsilon^{2}|| B \|_{F}^{2}\leq	\log |\operatorname{det}(I+\epsilon B)| \leq \epsilon \operatorname{tr}(B)- 2\epsilon^{2}|| B \|_{F}^{2}.
	$$
\end{lemma}

\begin{proof}[\textbf{proof of \Cref{lem:2}}]
	We follow the proof from \cite{liu2017stein}. When $\epsilon<\frac{1}{\varrho\left(B+B^{\top}\right)}$, we have $$\varrho\left(I+\epsilon\left(B+B^{\top}\right)+\epsilon^2 BB^{\top}\right) \geq 1-\epsilon \varrho\left(B+B^{\top}\right)>0,$$ and so $I+\epsilon\left(B+B^{\top}\right)+\epsilon^2BB^{\top}$ is positive definite.
	By the property of matrix determinant, we have
	\begin{equation}\label{eq:rom1}
		\begin{aligned}
			\log |\operatorname{det}(I+\epsilon B)| &=\frac{1}{2} \log \operatorname{det}\left((I+\epsilon B)(I+\epsilon B)^{\top}\right) \\
			&=\frac{1}{2} \log \operatorname{det}\left(I+\epsilon\left(B+B^{\top}\right)+\epsilon^{2} B B^{\top}\right) \\
			& = \frac{1}{2} \log \operatorname{det}\left(I+\epsilon\left(B+B^{\top}+\epsilon B B^{\top}\right)\right).
		\end{aligned}
	\end{equation}
	
	Let $A=B+B^{\top}+\epsilon B B^{\top}$, we can establish
	$$
	\epsilon \operatorname{tr}(A)-\frac{\epsilon^{2}}{2} \frac{\|A\|_{F}^{2}}{1-\epsilon \varrho(A)}\leq\log \operatorname{det}(I+\epsilon A) \leq \epsilon \operatorname{tr}(A)-\frac{\epsilon^{2}}{2} \frac{\|A\|_{F}^{2}}{1+\epsilon \varrho(A)},
	$$
	which holds for any symmetric matrix $A$ and $0 \leq \epsilon<1 / \varrho(A)$. This is because, assuming $\left\{\lambda_{i}\right\}$ are the eigenvalues of $A$,
	$$
	\begin{aligned}
		\log \operatorname{det}(I+\epsilon A)-\epsilon \operatorname{tr}(A) &=\sum_{i}\left[\log \left(1+\epsilon \lambda_{i}\right)-\epsilon \lambda_{i}\right] \\
		&=\sum_{i}\left[\int_{0}^{1} \frac{\epsilon \lambda_{i}}{1+s \epsilon \lambda_{i}} \mathrm{~d} s-\epsilon \lambda_{i}\right] \\
		&=-\sum_{i} \int_{0}^{1} \frac{s \epsilon^{2} \lambda_{i}^{2}}{1+s \epsilon \lambda_{i}} \mathrm{~d} s,
	\end{aligned}
	$$
	while 
	\begin{eqnarray*}
		-\frac{\epsilon^2}{2}\frac{\norm{A}_{F}}{1-\epsilon\varrho(A)}	&=& -\frac{1}{2}\sum_i\frac{\epsilon^2\lambda_{i}^2}{1-\epsilon\max_{i}|\lambda_{i}|} \\
		&\leq & -\sum_{i} \int_{0}^{1} \frac{s \epsilon^{2} \lambda_{i}^{2}}{1+s \epsilon \lambda_{i}} \mathrm{~d} s \\
		&\leq & 	-\frac{1}{2}\sum_i\frac{\epsilon^2\lambda_{i}^2}{1+\epsilon\max_{i}|\lambda_{i}|} \\
		&=& -\frac{\epsilon^2}{2}\frac{\norm{A}_{F}}{1+\epsilon\varrho(A)},
	\end{eqnarray*}
	so we have 
	\begin{equation}\label{eq:romrom}
		-\frac{\epsilon^2}{2}\frac{\norm{A}_F}{1-\epsilon\varrho(A)}\leq	\log \operatorname{det}(I+\epsilon A)-\epsilon \operatorname{tr}(A)\leq -\frac{\epsilon^2}{2}\frac{\norm{A}_F}{1+\epsilon\varrho(A)}.
	\end{equation}
	
	Taking $A=B+B^{\top}+\epsilon BB^{\top}$ into \Cref{eq:romrom} and combine it with \Cref{eq:rom1}, we get
	$$
	\begin{aligned}
		\log |\operatorname{det}(I+\epsilon B)| & \geq \frac{1}{2} \log \operatorname{det}\left(I+\epsilon\left(B+B^{\top}+\epsilon BB^{\top}\right)\right) \\
		& \geq \frac{\epsilon}{2} \operatorname{tr}\left(B+B^{\top}+\epsilon BB^{\top}\right)-\frac{\epsilon^{2}}{4} \frac{\left\|B+B^{\top}+\epsilon BB^{\top}\right\|_{F}^{2}}{1-\epsilon \varrho\left(B+B^{\top}+\epsilon BB^{\top}\right)} \\
		& \geq \epsilon \operatorname{tr}(B)-\frac{\epsilon^{2}}{4} \left(\frac{9\|B\|_{F}^{2}}{1-\epsilon \varrho\left(B+B^{\top}+\epsilon BB^{\top}\right)}+2{\norm{B}^2_F}\right),
	\end{aligned}
	$$
	similarly
	\begin{equation*}
		\log |\operatorname{det}(I+\epsilon B)| \leq \epsilon \operatorname{tr}(B)-\frac{\epsilon^{2}}{4} \left(\frac{9\|B\|_{F}^{2}}{1+\epsilon \varrho\left(B+B^{\top}+\epsilon BB^{\top}\right)}+2{\norm{B}^2_F}\right)
	\end{equation*}
	where we used the fact that $\operatorname{tr}(B)=\operatorname{tr}\left(B^{\top}\right)$ , $\norm{BB^{\top}}_F\leq\norm{B}_F^2$  and $\left\|B+B^{\top}+\epsilon BB^{\top}\right\|_{F} \leq\|B\|_{F}+\left\|B^{\top}\right\|_{F}+\epsilon\norm{ BB^{\top}}_F=3\|B\|_{F}$~(since $\epsilon\leq\frac{1}{\norm{B}_F}$). Finally we use inequality $\varrho\left(B+B^{\top}+\epsilon BB^{\top}\right)\leq\varrho\left(B+B^{\top}\right)+\epsilon\varrho\left(BB^{\top}\right)\leq \varrho\left(B+B^{\top}\right)+\sqrt{\varrho\left(BB^{\top}\right)}$ and 
	\begin{equation*}
		\begin{aligned}
			\varrho(B+B^{\top})^2&\leq \operatorname{tr}\left(BB+BB^{\top}+B^{\top}B+B^{\top}B^{\top}\right)\\
			&=\operatorname{tr}(BB)+\operatorname{tr}(B^{\top}B^{\top})+2\operatorname{tr}(BB^{\top})\\
			&\leq 4\operatorname{tr}(BB^{\top})\qquad\qquad //\text{since}~\operatorname{tr}(BB)\leq\operatorname{tr}(BB^{\top})\\
			&=4\norm{B}_{F}^2
		\end{aligned}
	\end{equation*}
	and $\varrho(BB^{\top})\leq \norm{B}_{F}^2$, so we have 
	\begin{equation}
		\varrho\left(B+B^{\top}+\epsilon BB^{\top}\right)\leq 3\norm{B}_F.
	\end{equation}
	Combining all of these, we finally get
	\begin{equation}
		\begin{aligned}
			&\epsilon \operatorname{tr}(B)-\frac{\epsilon^{2}}{4} \left(\frac{9\|B\|_{F}^{2}}{1-3\epsilon \norm{B}_F}+2{\norm{B}^2_F}\right)\\
			&\leq	\log |\operatorname{det}(I+\epsilon B)| \\
			&\leq \epsilon \operatorname{tr}(B)-\frac{\epsilon^{2}}{4} \left(\frac{9\|B\|_{F}^{2}}{1+3\epsilon \norm{B}_F}+2{\norm{B}^2_F}\right).
		\end{aligned}
	\end{equation}
	
\end{proof}

\section{Experiments}\label{apdx:exp}
The code can be found in \url{https://github.com/Iwillnottellyou/BETA-SVGD.git}.

\subsection{Gaussian Mixtures}
In \Cref{fig:9}, \Cref{fig:4}, \Cref{fig:5} and \Cref{fig:6}, we use Gaussian Mixtures to test the performance of \algname{$\beta$-SVGD}.  We choose the reproducing kernal  $k(x,y)=e^{-\frac{\normsq{x-y}}{d}}$, where $d$ is the dimension. 

\begin{figure}
	\centering
	\includegraphics[scale=0.45]{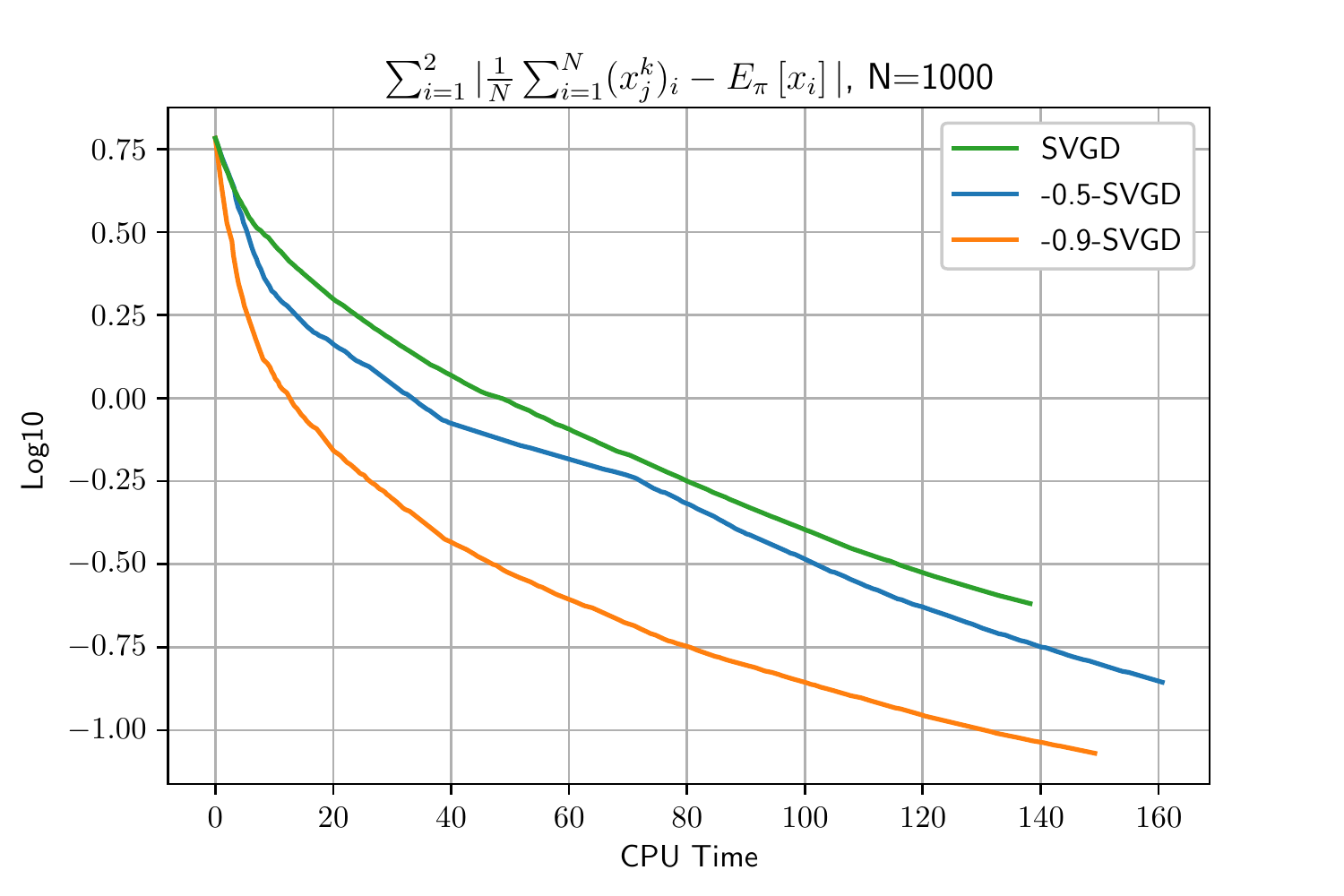}
	\includegraphics[scale=0.45]{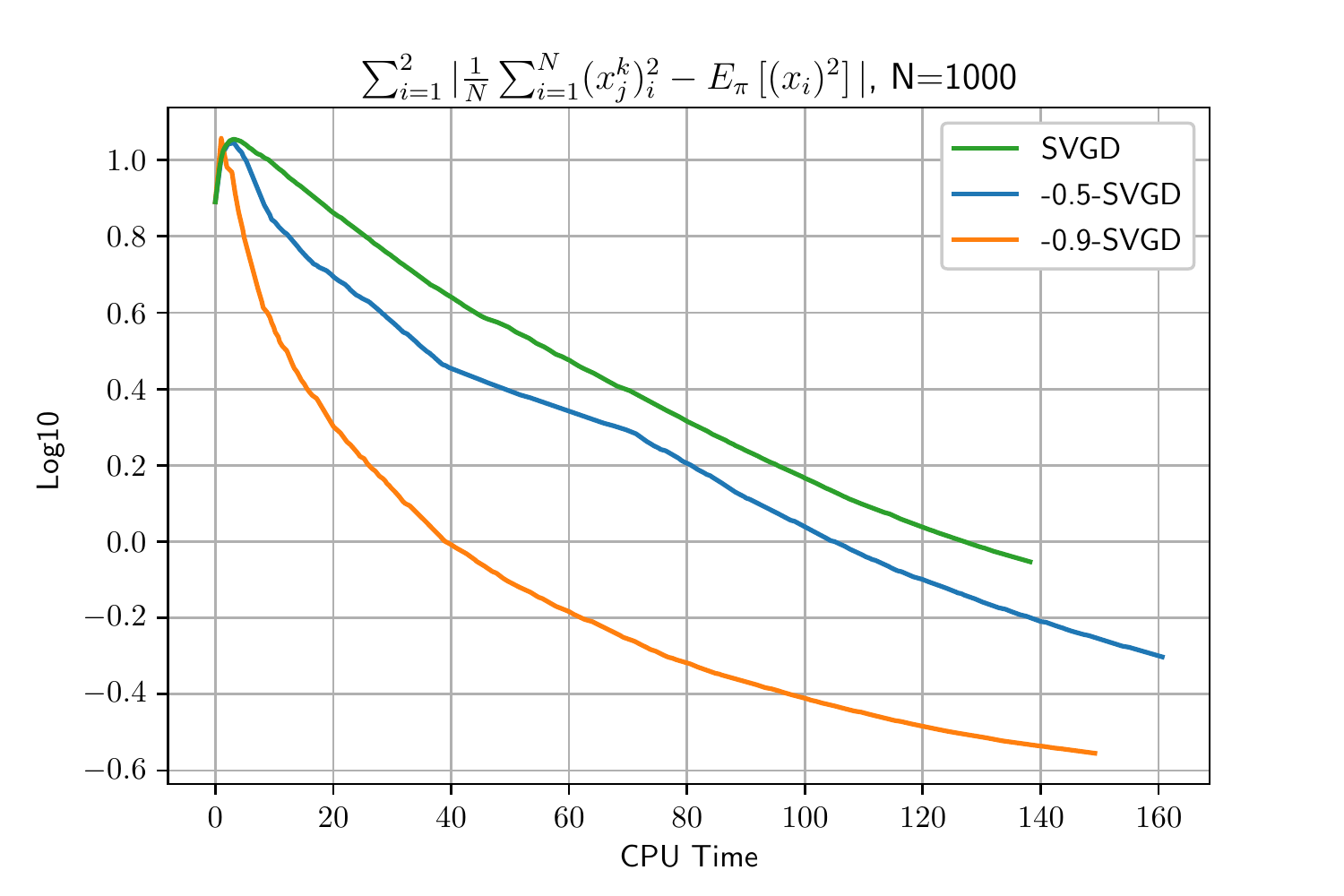}
	\caption{The target distribution is $\pi(x)=\frac{2}{5}\mathcal{N}((2,0),I_2)+\frac{1}{5}\mathcal{N}((4,0),I_2)+\frac{2}{5}\mathcal{N}((3,-3),I_2)$. Each sampled point  $x_j^k$ is of the form $\left((x_j^k)_1,(x_j^k)_2\right)$, where $k$ denote the $k$-th iteration, $j$ denote the $j$-th sampled point. For distribution $\pi$, we have $\mathbb{E}_{\pi}\left[x_1\right]=2.8,\mathbb{E}_{\pi}\left[x_2\right]=-1.2$ and $\mathbb{E}_{\pi}\left[(x_1)^2\right]=9.4,\mathbb{E}_{\pi}\left[(x_2)^2\right]=4.6$. The initial $N$ points are sampled from $\mathcal{N}\left((-2,0),I_2\right)$. The step-size $\gamma$ for both algorithms equals $0.2$. In \algname{$\beta$-SVGD}, we choose the small gap $\tau=0.01$ and we update the Stein importance weights every $20$ iterations using $40$ mirror descent steps with step-size $r=0.3$. Since the function computed in the second image is $x^2$, it is not surprising that there is an increase in the first few iterations.}
	\label{fig:9}
\end{figure}

\begin{figure}
	\hspace{-2.7cm}	\includegraphics[scale=0.45]{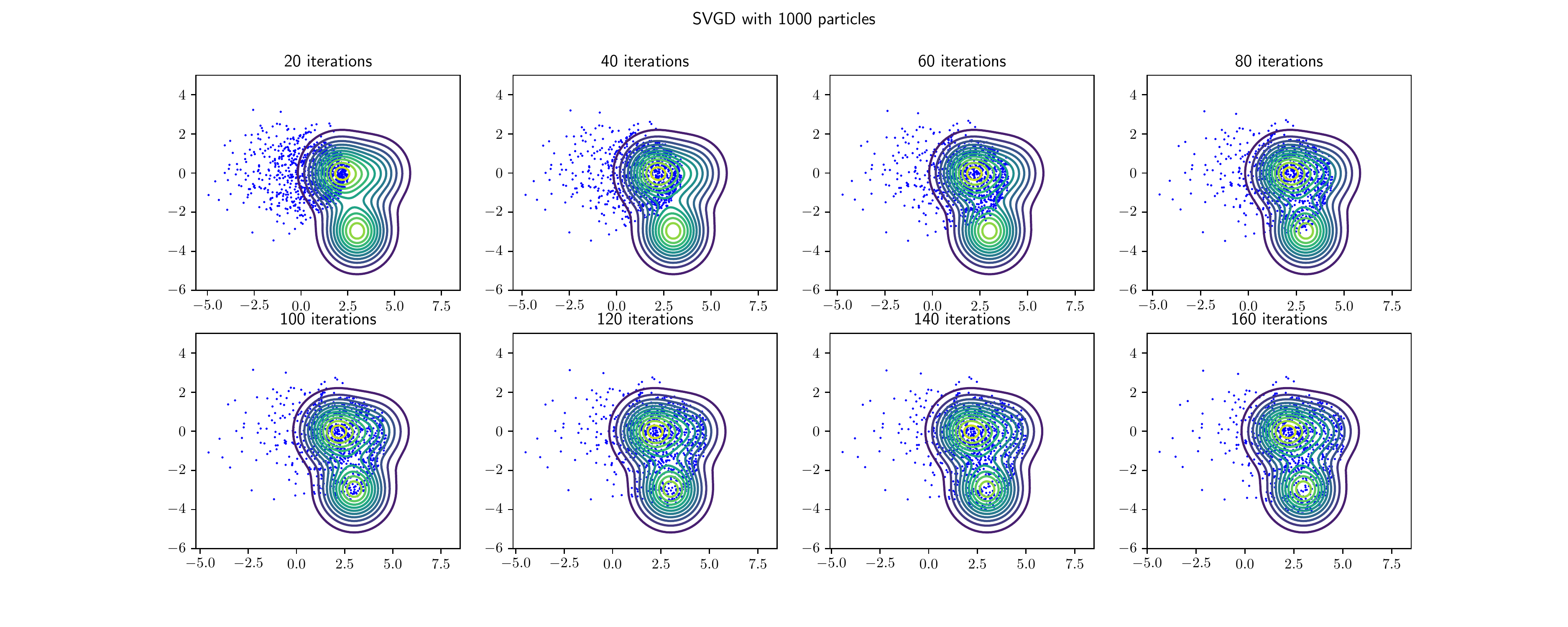}
	
	\hspace{-2.7cm}	\includegraphics[scale=0.45]{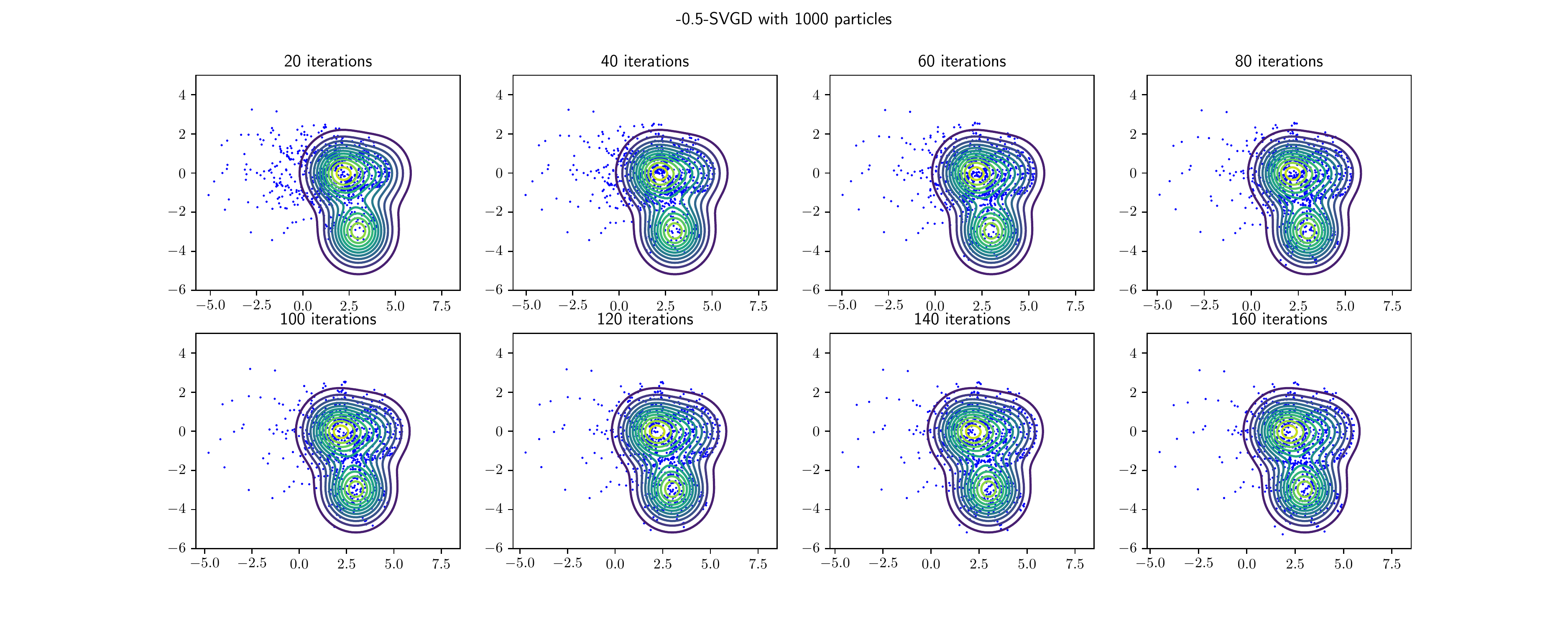}
	
	\hspace{-2.7cm}	\includegraphics[scale=0.45]{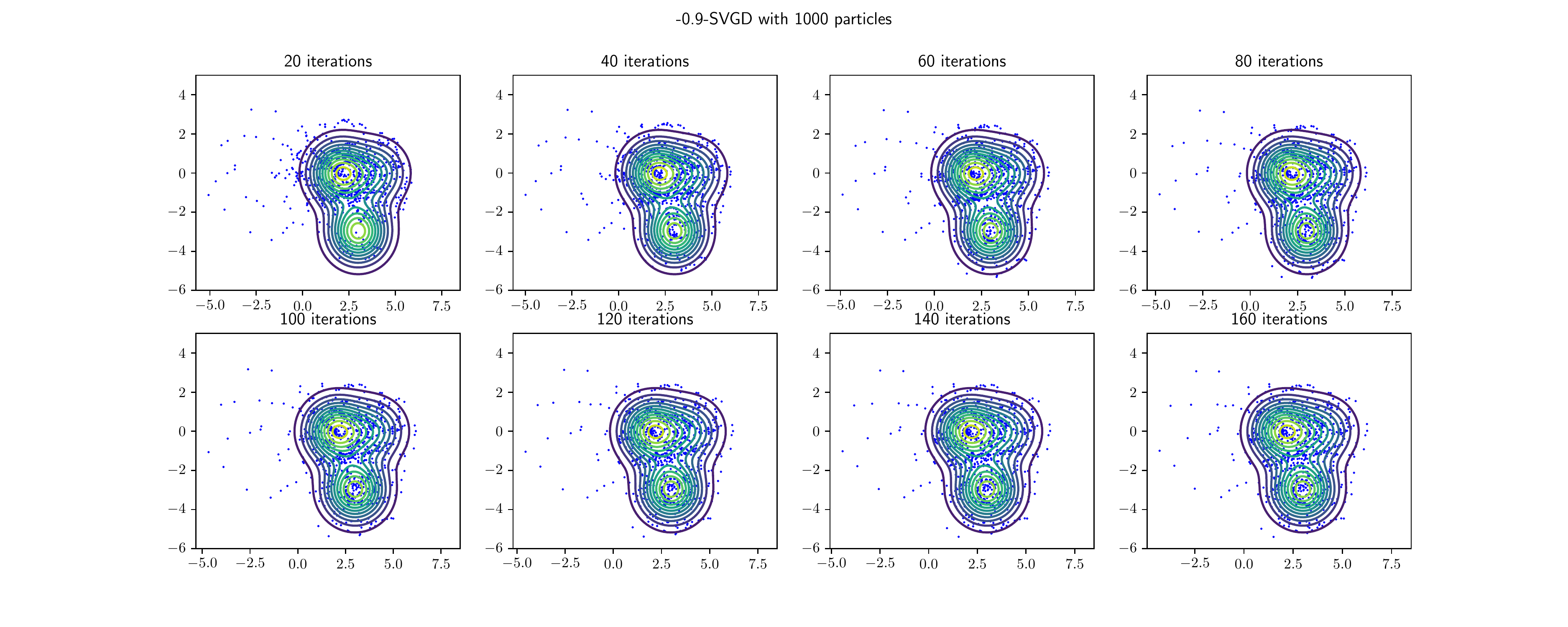}
	
	\caption{
		The same experiment setting as in \Cref{fig:3}. We show how the particles move in the update of \algname{$\beta$-SVGD} with $\beta=0,-0.5,-0.9$. }
	\label{fig:4}
\end{figure}

\begin{figure}
	\centering
	\includegraphics[scale=0.45]{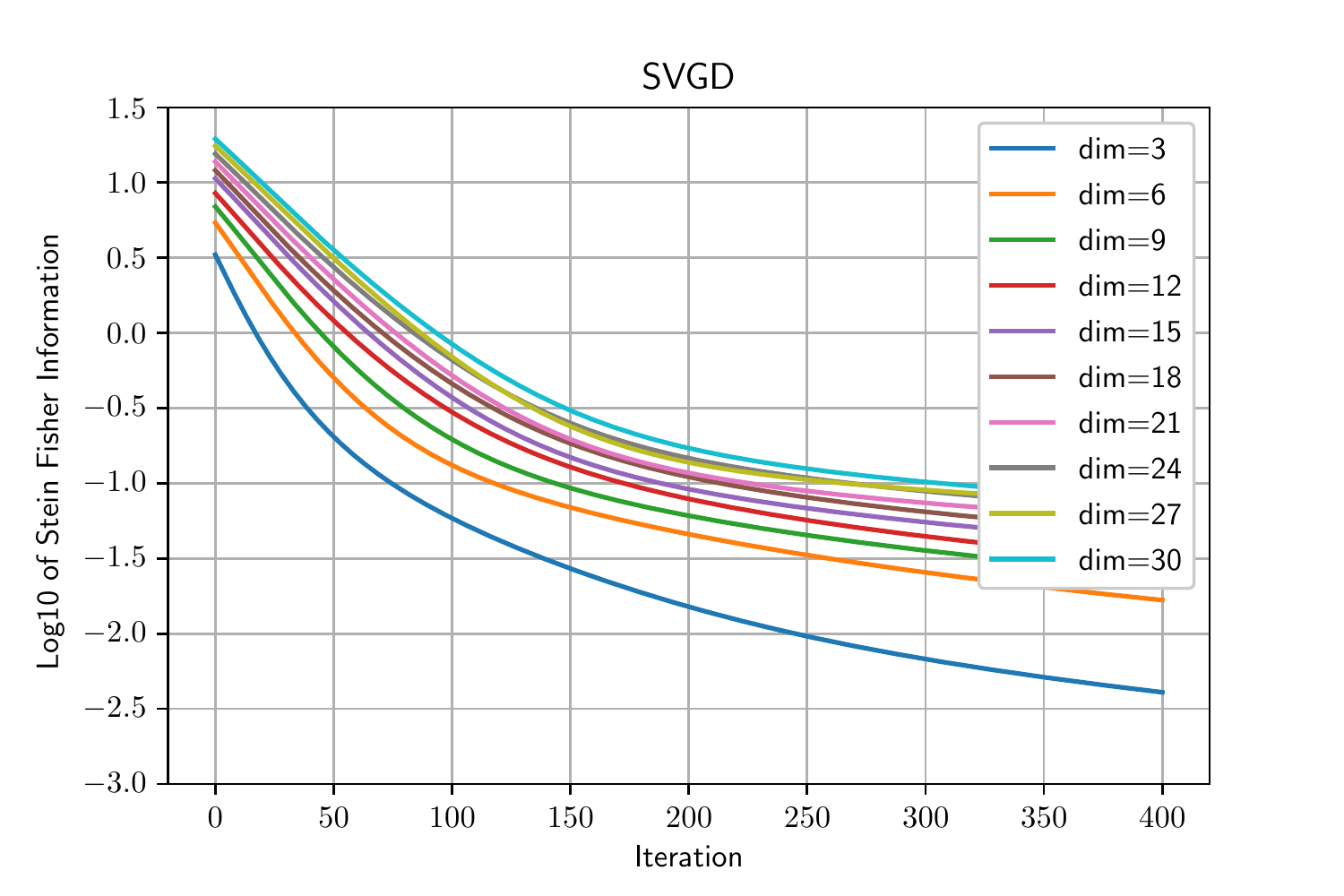}
	\includegraphics[scale=0.45]{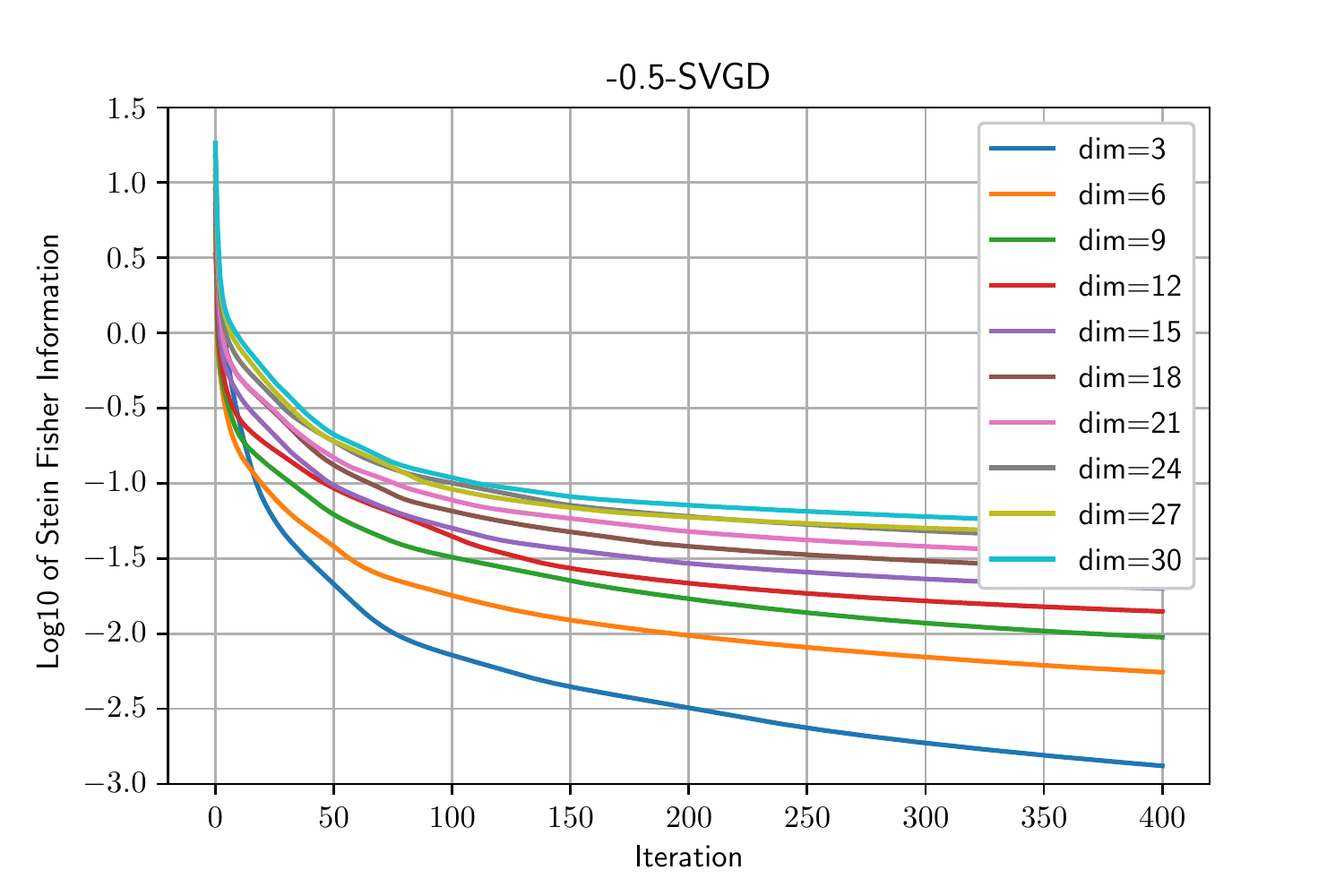}
	\caption{In this experiment, we show how the Stein Fisher information changes in the update of \algname{SVGD} and \algname{$-0.5$-SVGD}. The target distribution is $\mathcal{N}((2,\ldots,2)_d,I_d)$ and the initial points are sampled from $\mathcal{N}((0,\ldots,0)_d,I_d)$ with $N=300$. The step-size $\gamma=0.1$ for both algorithm and for \algname{$-0.5$-SVGD} algorithm, we set the small gap $\tau=0.01$ and we update the Stein importance weight in every iteration using $40$ mirror descent with step-size $r=0.3$. We can see that the Stein Fisher information drops immediately below $1$~(note in the picture, the axis $y$ is $\log_{10}$ of the Stein Fisher information) in \algname{$-0.5$-SVGD}, while in \algname{SVGD} it drops slowly.
	}
	\label{fig:5}
\end{figure}

\begin{figure}
	\centering
	\includegraphics[scale=0.45]{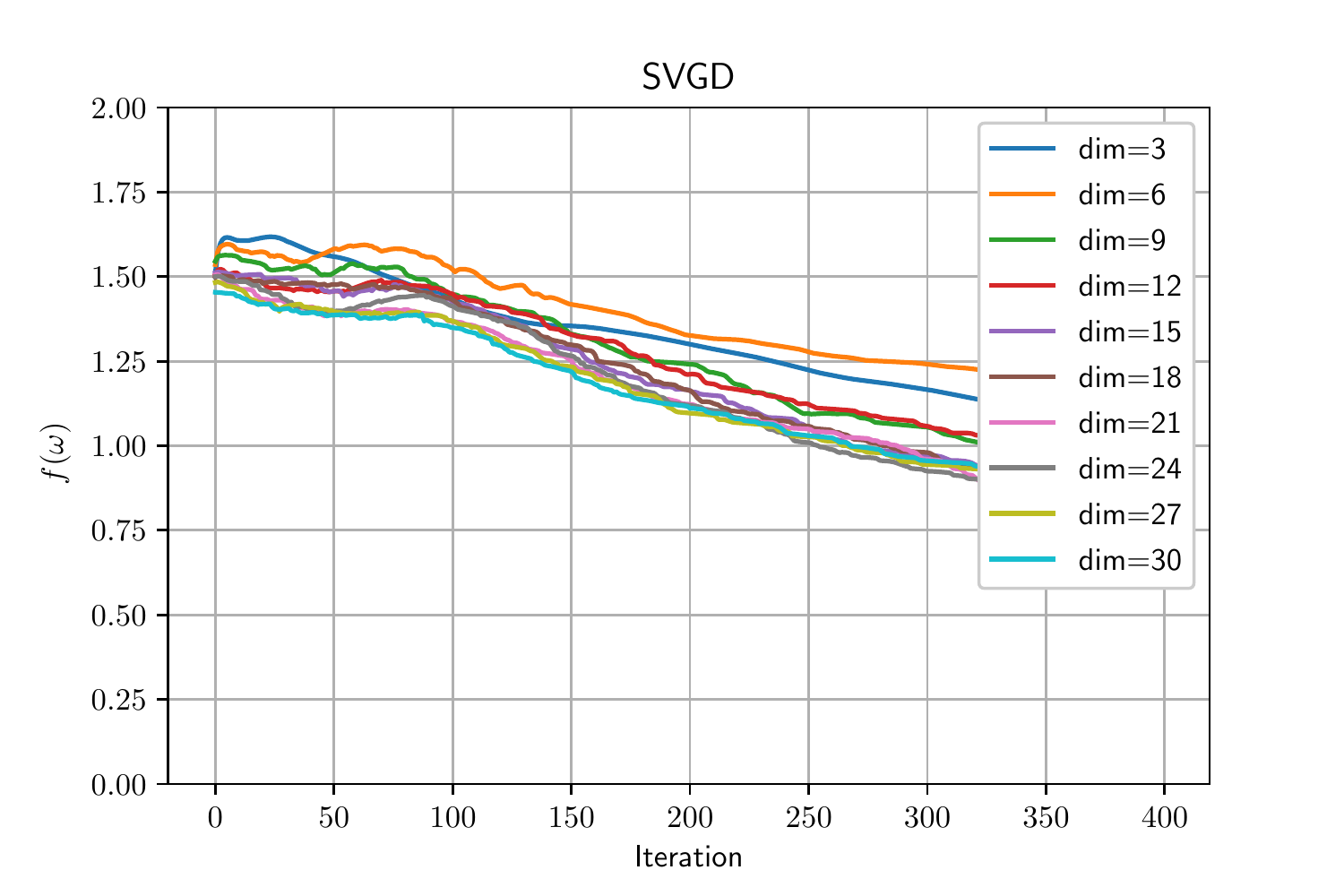}
	\includegraphics[scale=0.45]{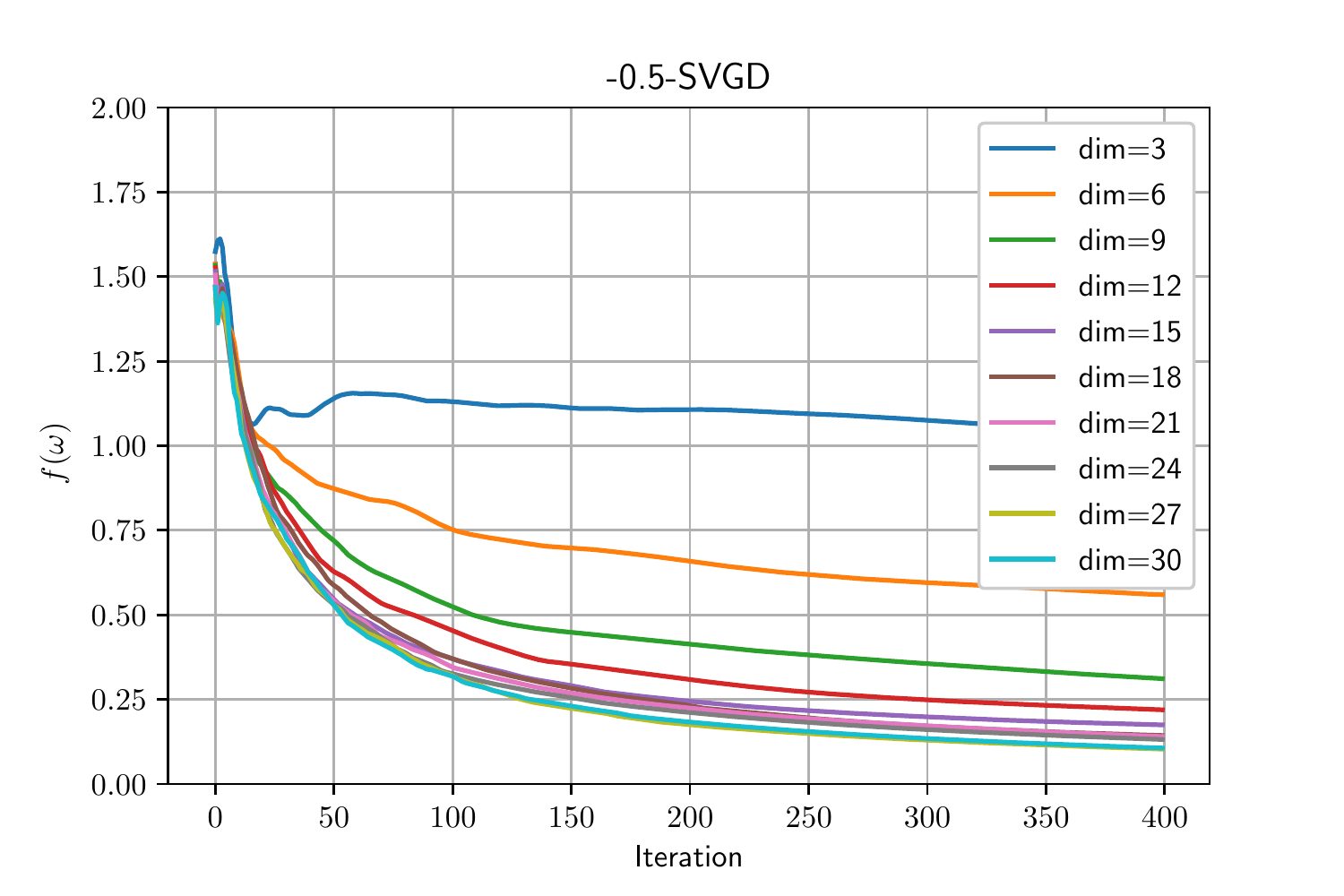}
	\caption{The experiment settings are the same as in \Cref{fig:5}. We compare how the Stein importance weight changes in the update of \algname{SVGD} and \algname{$-0.5$-SVGD}~(though we don't have to compute the Stein importance weight in the implementation of \algname{SVGD}). The error is defined by $f(\omega^k):=\sum_{i=1}^{N}|w_i^k-\frac{1}{N}|$, where $\omega_i^k$ denote the Stein importance weight of point $x_i^k$ and $N=300$. The results suggest that in high dimensional cases,  the Stein importance weight can help to accelerate the decreasing of Stein Fisher information in the beginning, then it will approach to the identical weight $\frac{1}{N}$ quickly.}
	\label{fig:6}
\end{figure}

\subsection{Bayesian Logistic Regression}
In \Cref{fig:7}, we compare the performance of \algname{SVGD} and \algname{$\beta$-SVGD} with $\beta=-0.5$ in Bayesian Logistic regression problem. This Bayesian Logistic regression experiment is done in \cite{liu2016stein} to compare \algname{SVGD} with several Markov Chain Monte Carlo methods, more details about this experiment can refer to \cite{liu2016stein}. As in the Gaussian Mixtures experiment, we choose the reproducing kernal $k(x,y)=e^{-\frac{\normsq{x-y}}{d}}$.

\begin{figure}
	\centering
	\includegraphics[scale=0.4]{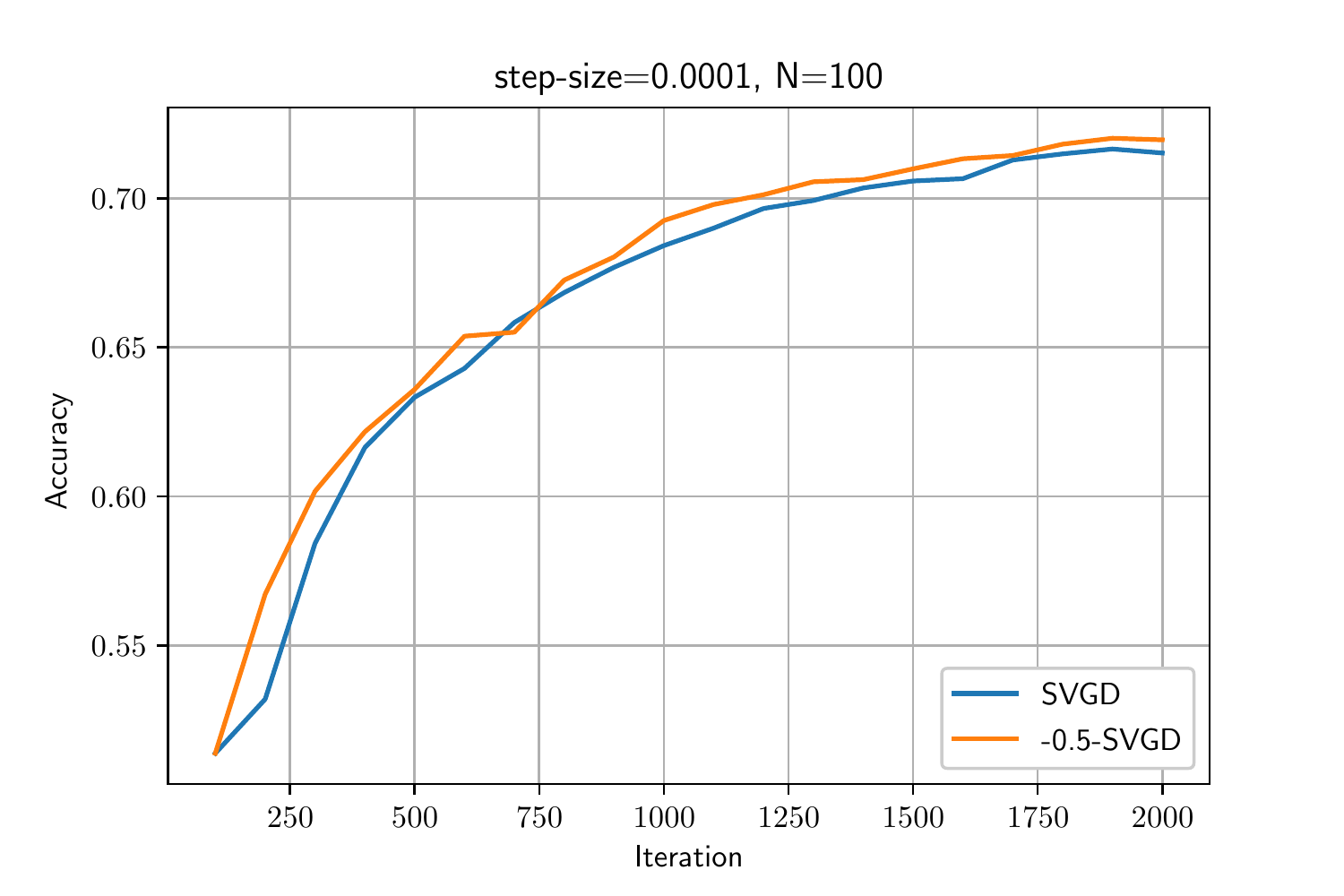}
	\includegraphics[scale=0.4]{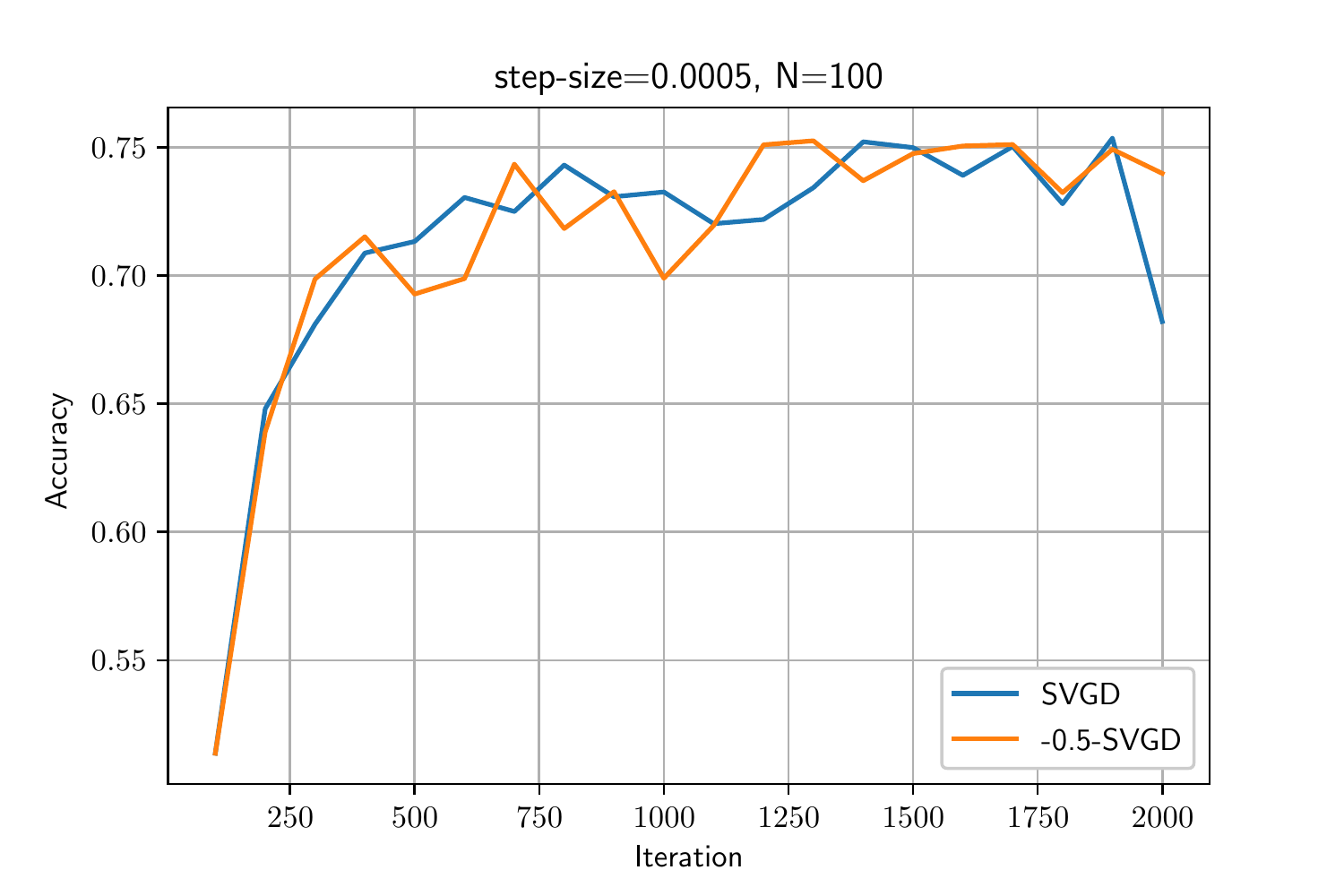}
	\includegraphics[scale=0.4]{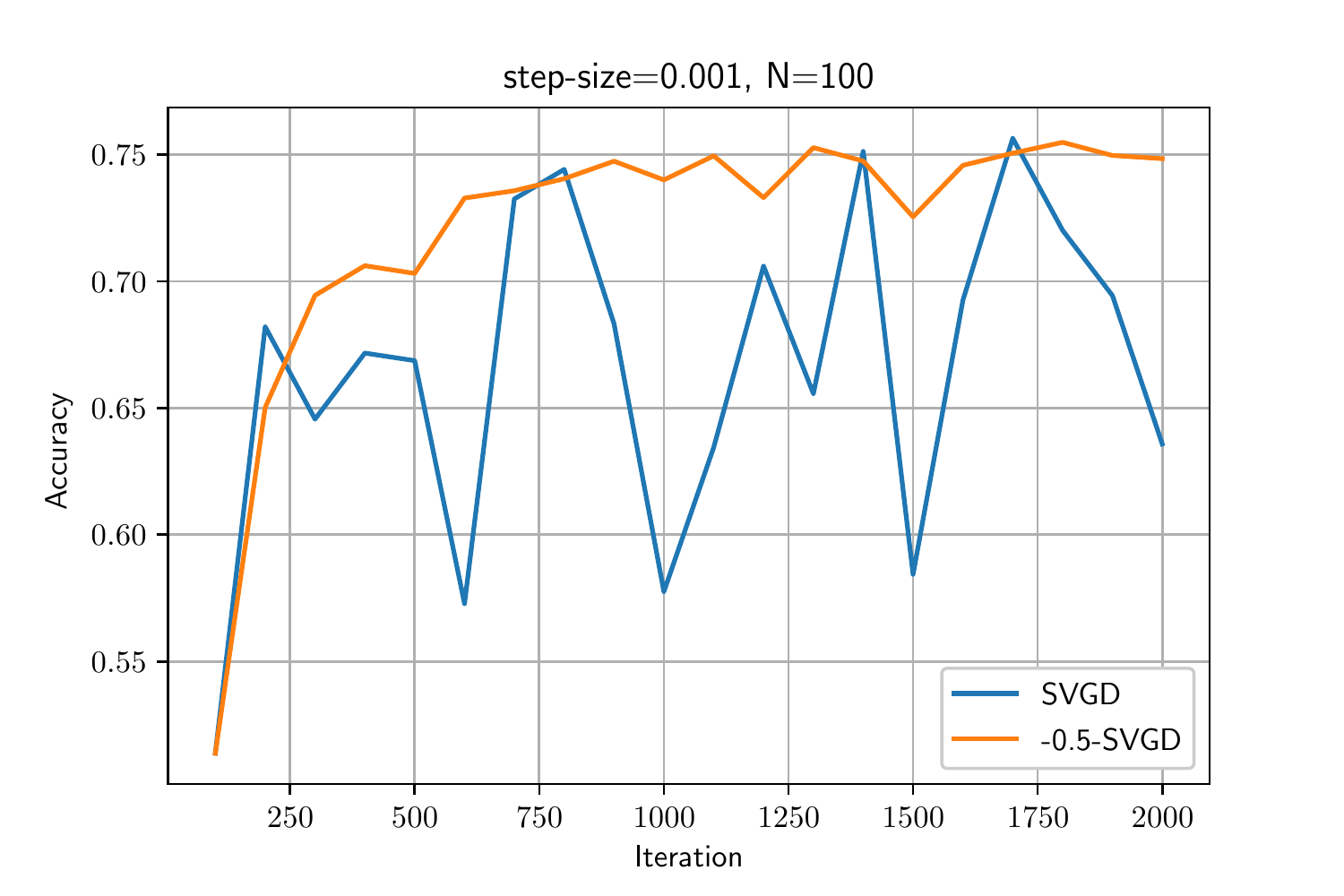}
	\includegraphics[scale=0.4]{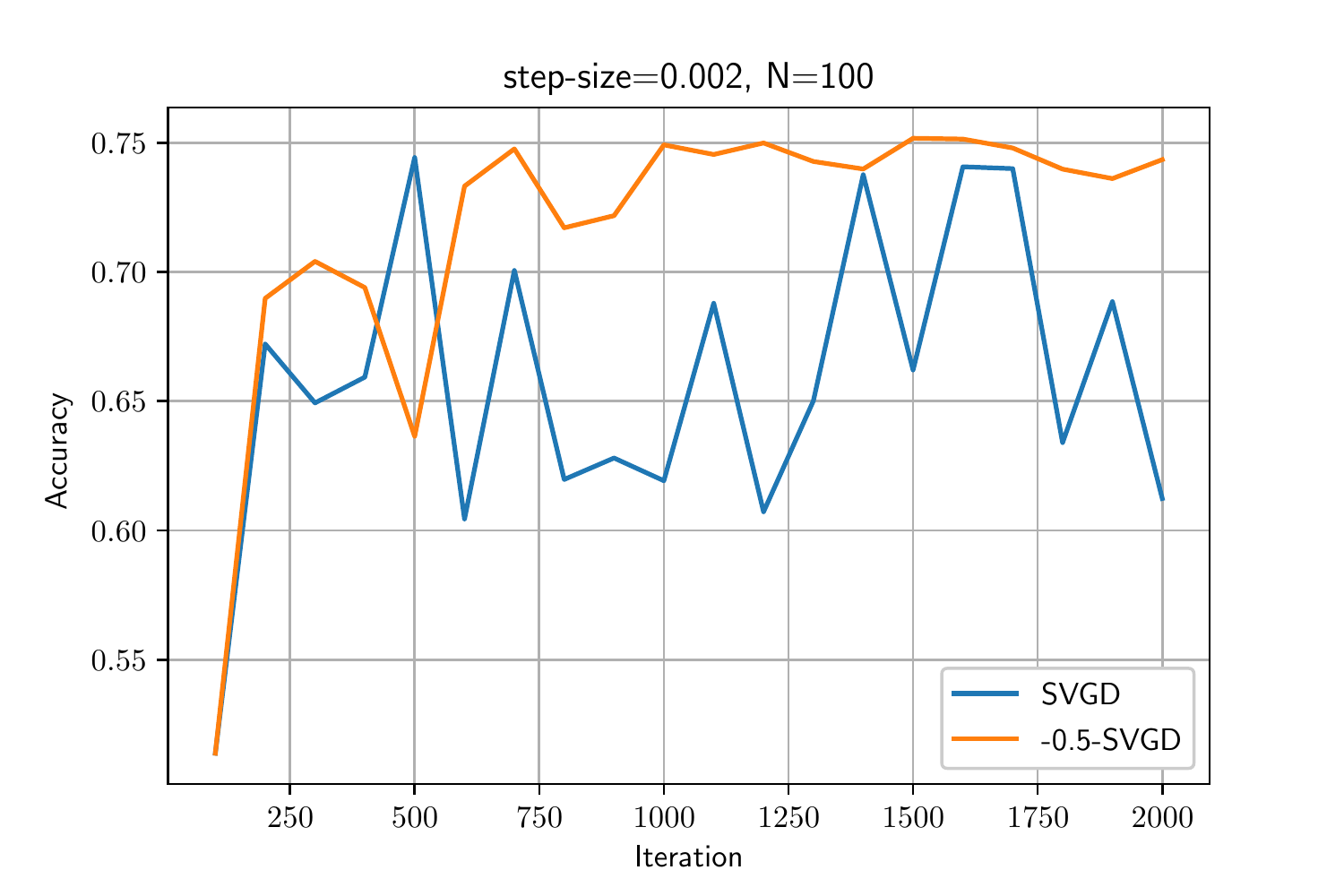}
	\includegraphics[scale=0.4]{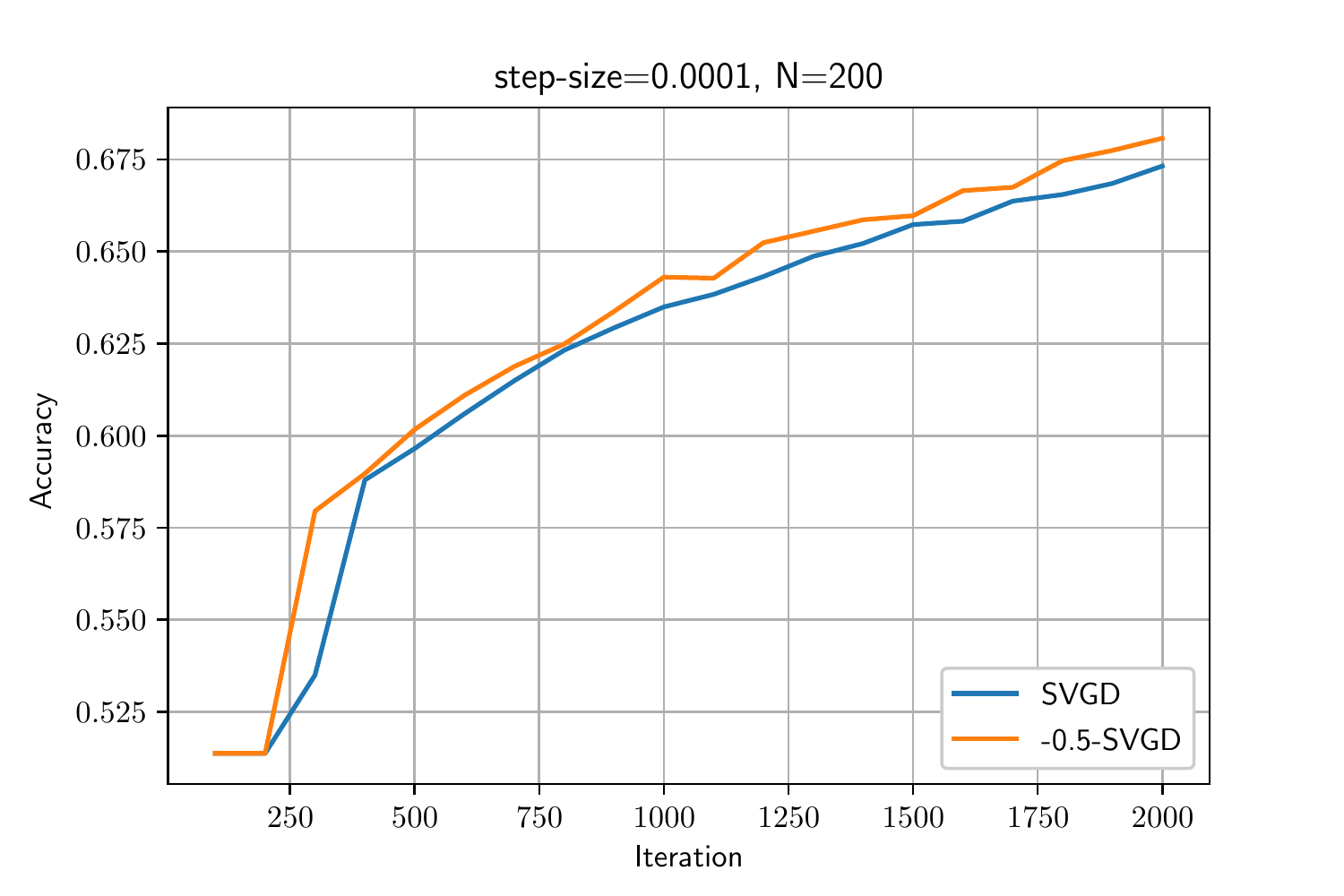}
	\includegraphics[scale=0.4]{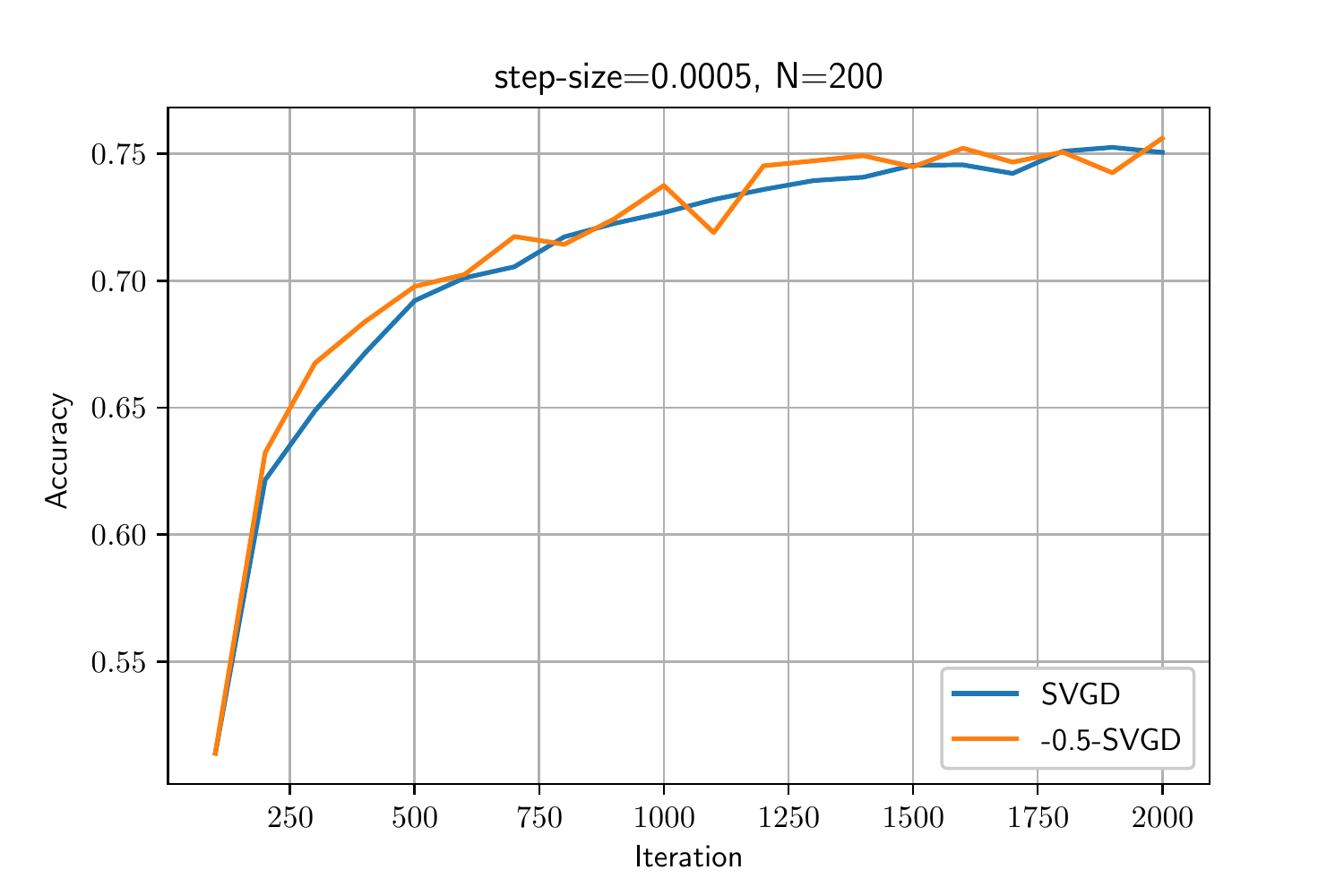}
	\includegraphics[scale=0.4]{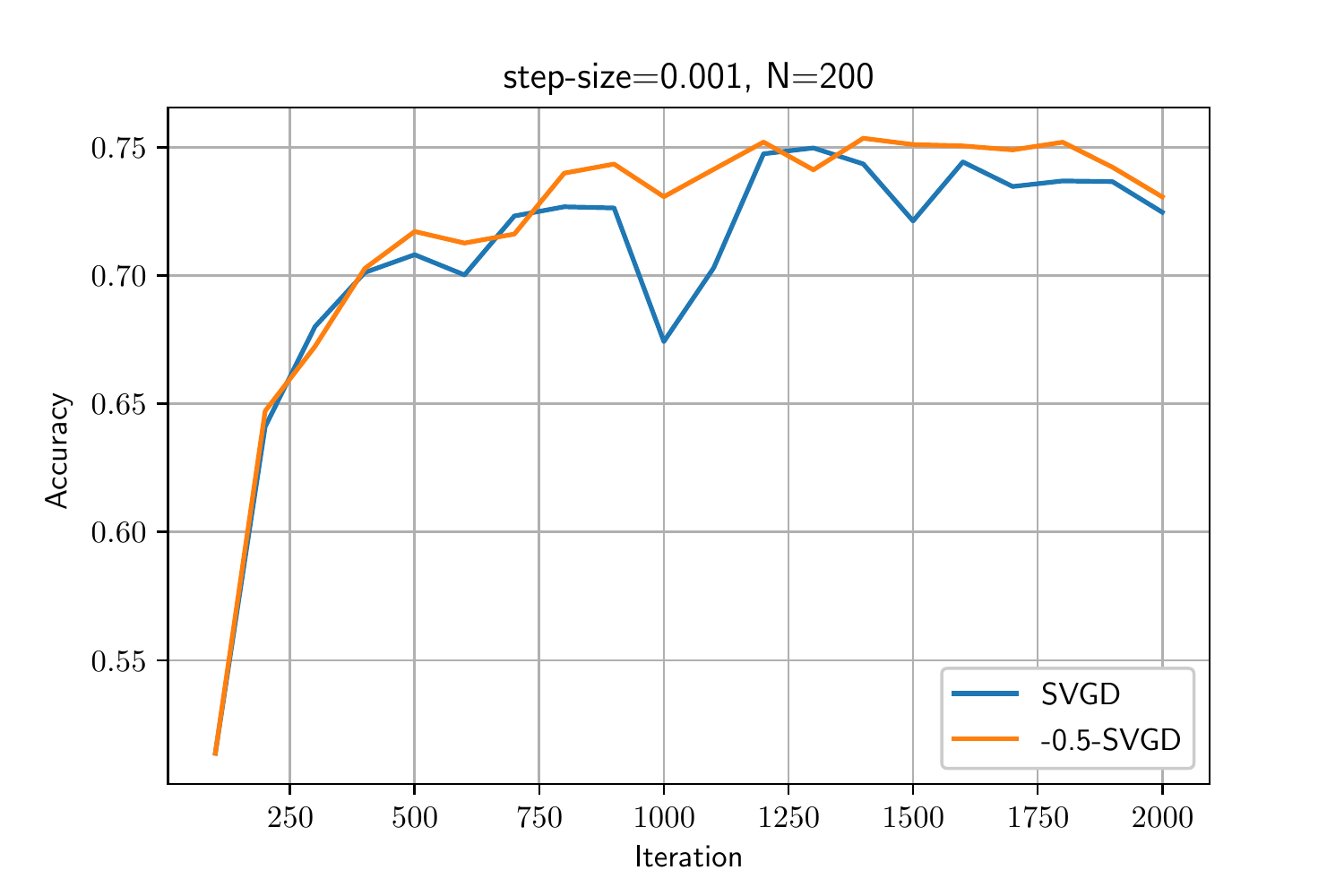}
	\includegraphics[scale=0.4]{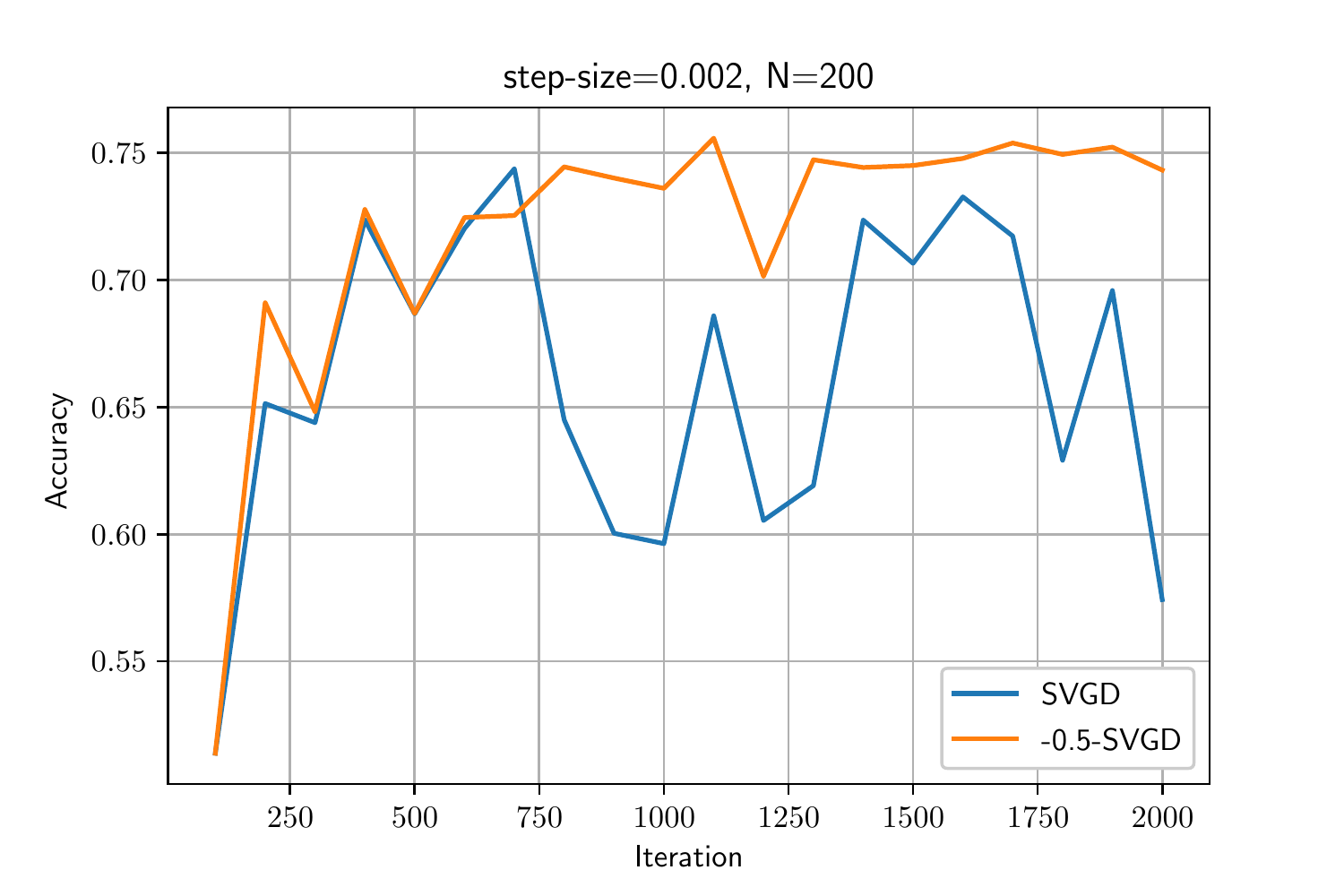}
	\caption{ In this experiment, we test the binary Covertype dataset with 581,012 data points and 54 features~($d=54$). We run $2000$ iterations of \algname{SVGD} and \algname{$\beta$-SVGD} with different step-size and number of particles. In each iteration of \algname{$-0.5$-SVGD}, we run $200$ steps of mirror descent with step-size $r=2$~(since the values of the entries of $\boldsymbol{K}_{\pi}$ in this experiment can be very big, we need to rescale the matrix by dividing a factor of $10^9$ to resolve the overflow problem, so the step-size for mirror descent is chosen relatively big) to find the Stein importance weights, we set the small gap $\tau=0.05$. The time required to run $2000$ iterations of \algname{$-0.5$-SVGD} and test the accuracy every $100$ iterations is roughly double that required for \algname{SVGD}. In this experiment, we found the Stein importance weight is close to the identical weight after only a few \algname{$\beta$-SVGD} iterations with relatively big step-size~(specifically, the percentage of weight $\omega_i$ such that $N\omega_i<0.1$ falls to $0$ after the first few iterations of \algname{$-0.5$-SVGD}), so the acceleration effect is not very clear in this case. However, as shown in the first and fifth images, where the step-size is relatively small,  we can see a faster improvement in accuracy in the first few hundreds iterations of \algname{$\beta$-SVGD}. We can also see from the results that when $\gamma$ is relatively large, due to the Stein importance weight, \algname{$-0.5$-SVGD} is much more stable than \algname{SVGD}.}
	\label{fig:7}
\end{figure}

\end{document}